\documentclass{article} 
\usepackage{collas2024_conference,times}
\usepackage{easyReview}

\usepackage{graphicx}
\usepackage{booktabs}
\usepackage{multirow}
\usepackage{comment}
\usepackage{algorithm}
\usepackage{algorithmic}
\usepackage{wrapfig,lipsum,booktabs}
\usepackage{float}
\usepackage{subfigure}
\usepackage{capt-of}
\usepackage{mylatexstyle}

\usepackage{amsmath,amsfonts,bm}

\def\eqref#1{equation~\ref{#1}}

\def\1{\bm{1}}

\def\vb{{\bm{b}}}

\DeclareMathAlphabet{\mathsfit}{\encodingdefault}{\sfdefault}{m}{sl}
\SetMathAlphabet{\mathsfit}{bold}{\encodingdefault}{\sfdefault}{bx}{n}

\newcommand{\R}{\mathbb{R}}

\usepackage{dsfont}
\usepackage{amsthm}

\newcommand{\vertexSet}[0]{\mathcal{V}}

\newcommand{\prob}[0]{\mathbb{P}}

\newcommand{\expect}[0]{\mathds{E}}

\newcommand{\loss}[0]{{L}}
\newcommand{\gnnModel}[0]{g}

\newcommand{\graphStruct}[0]{\mathcal{G}}
\newcommand{\neighbor}[1]{\mathcal{N}_{#1}}

\newcommand{\ego}[1]{\mathbf{G}_{#1}}

\newcommand{\task}[0]{\tau}

\newcommand{\cataForget}[0]{\textbf{CF}}
\newcommand{\hypothesis}[0]{\mathcal{F}}

\newcommand{\jointDistCondVertSet}[2]{\prob(\mathbf{y}_{#1},\ego{#1}|#2)}

\newcommand{\experienceBuffer}[0]{\mathcal{P}}

\newtheorem{corollary}[theorem]{Corollary}

\usepackage{hyperref}
\usepackage{url}
\usepackage{enumitem}

\usepackage{hyperref}
\hypersetup{
    colorlinks=true,
    linkcolor=red,
    filecolor=magenta,
    urlcolor=blue,
    citecolor=purple,
    pdftitle={Overleaf Example},
    pdfpagemode=FullScreen,
}

\title{On the Limitation and Experience Replay for GNNs in Continual Learning}

\author{Junwei Su\\
Department of Computer Science\\
The University of Hong Kong\\
Hong Kong \\
\texttt{jwsu@cs.hku.hk} \\
\And 
Difan Zou  \\
Department of Computer Science\\
The University of Hong Kong\\
Hong Kong \\
\texttt{dzou@cs.hku.hk} \\
\AND 
Chuan Wu \\
Department of Computer Science\\
The University of Hong Kong\\
Hong Kong \\
\texttt{cwu@cs.hku.hk} \\
}

\collasfinalcopy 

\begin{document}

\maketitle

\begin{abstract}
Continual learning seeks to empower models to progressively acquire information from a sequence of tasks. This approach is crucial for many real-world systems, which are dynamic and evolve over time. Recent research has witnessed a surge in the exploration of Graph Neural Networks (GNN) in Node-wise Graph Continual Learning (NGCL), a practical yet challenging paradigm involving the continual training of a GNN on node-related tasks.  Despite recent advancements in continual learning strategies for GNNs in NGCL, a thorough theoretical understanding, especially regarding its learnability, is lacking. Learnability concerns the existence of a learning algorithm that can produce a good candidate model from the hypothesis/weight space, which is crucial for model selection in NGCL development. This paper introduces the first theoretical exploration of the learnability of GNN in NGCL, revealing that learnability is heavily influenced by structural shifts due to the interconnected nature of graph data. Specifically,  GNNs may not be viable for NGCL under significant structural changes, emphasizing the need to manage structural shifts.  To mitigate the impact of structural shifts, we propose a novel experience replay method termed Structure-Evolution-Aware Experience Replay (SEA-ER). SEA-ER features an innovative experience selection strategy that capitalizes on the topological awareness of GNNs, alongside a unique replay strategy that employs structural alignment to effectively counter catastrophic forgetting and diminish the impact of structural shifts on GNNs in NGCL. Our extensive experiments validate our theoretical insights and the effectiveness of SEA-ER. 
\end{abstract}


\section{Introduction}\label{sec:introduc.}
Continual learning, also known as incremental learning or life-long learning, investigates machine learning approaches that enable a model to continuously acquire new knowledge while retaining previously obtained knowledge. This capability is essential as it allows the model to adapt to new information without forgetting past information. In the general formulation of continual learning, a stream of tasks arrives sequentially, and the model undergoes rounds of training sessions to accumulate knowledge for specific objectives, such as classification. The primary goal is to find a learning algorithm that can continually update the model's parameters based on the new task without suffering from \emph{catastrophic forgetting}~\citep{cl_survey}, which refers to the inability to retain previously learned information when learning new tasks. Continual learning is crucial for the practicality of machine learning systems as it enables the model to adapt to new information without the need for frequent retraining, which can be costly in the deep learning regime today.

On the other hand, graph neural networks (GNNs) have gained widespread popularity as effective tools for modeling graph and relational data structures~\citep{gnn_survey,su2024bg}. However, most studies on GNNs have predominantly focused on static settings, assuming that the underlying graph and learning task remain unchanged. This approach, while valuable for certain applications, does not reflect the dynamic and evolving nature of real-life networks. For example, networks such as citation networks~\citep{zliobaite2010learning} and financial networks~\citep{gama2014survey} naturally evolve over time, accompanied by the emergence of new tasks. Therefore, to further advance the practicality and effectiveness of GNNs, it is essential to study their effectiveness and applicability in graph continual learning (GCL).

Due to its practical significance, recent research has seen an upsurge in the investigation of the application of GNNs in GCL~\citep{febrinanto2023graph, yuan2023continual}. Nonetheless, there is a noticeable gap in the theoretical foundations, particularly regarding the node-level prediction task, commonly referred to as \emph{Node-wise Graph Continual Learning} (NGCL) \citep{su2023towards, zhang2022cglb}. In NGCL, the GNN model undergoes multiple training rounds (tasks) to accumulate knowledge for specific objectives, such as node classification. In this scenario, vertices across different tasks can be interconnected, and the introduction of new tasks may alter the structure of the existing graph, as illustrated in Figure~\ref{fig:NGCL}. Given that GNNs rely on graph structures as input, these structural changes can lead to shifts in the data distributions of previous tasks, a phenomenon known as \emph{structural shift}. \cite{su2023towards} has shown that the risk (upper bound) of the performance of GNNs in NGCL is closely related to structural shift. Nevertheless, it is not yet clear whether and how structural shifts can affect the learnability of GNNs in NGCL. Learnability concerns whether there exists a learning algorithm that can produce a good candidate model from the hypothesis/weight space, and the learnability of GNN in NGCL is crucial to NGCL development. This has motivated the central question of this study:
\begin{center}
\emph{Is GNN always learnable (a suitable model choice) for NGCL under structural shift?}
\end{center}

In this work, we address the aforementioned question by formally formulating and analyzing the learnability of GNNs in NGCL. To the best of our knowledge, this work presents the first theoretical studies on the learnability of GNNs in NGCL, and our first contribution is highlighted as follows:

\begin{itemize}[leftmargin=*]
\item   We present a mathematical formulation for studying the learnability of GNNs in NGCL. This formulation is important for a rigorous exploration of the learnability problem. In particular, we prove that a large structural shift can render GNNs unlearnable in NGCL, as articulated in Theorem~\ref{thm:divergence_necessity}. This implies that GNNs may not be a suitable choice for NGCL in volatile network systems (experiencing rapid changes in graph structure), answering the central question of this study. Moreover, this insight is particularly valuable for NGCL development in practice, as it can guide the model selection process by providing insights on when to select or rule out GNN-like model architectures for NGCL.
\end{itemize}

Given the pivotal role of structural shifts in learnability, we conclude the importance of controlling structural shifts even in scenarios where GNNs are learnable in NGCL. To address this challenge, we focus on experience replay, a principal method for continual learning, which has been empirically proven to be effective in addressing catastrophic forgetting in NGCL~\citep{ahrabian2021structure, cnc_er}. We introduce a novel experience replay method designed to mitigate the impact of structural shifts for GNNs in NGCL. Our contribution in this regard is as follows:

\begin{itemize}[leftmargin=*]
\item  We introduce Structure-Evolution-Aware Experience Replay (SEA-ER), a novel experience replay method that targets both catastrophic forgetting and the unique challenge of structural shifts in NGCL. SEA-ER leverages two key principles - structural alignment and topological awareness of GNNs - to select experience samples and assign different replaying weights based on their structural similarity to the rest of the graphs. This approach effectively mitigates structural shifts and catastrophic forgetting for GNNs in NGCL, offering theoretical assurances on long-term performance (Proposition~\ref{prop:performance_guarantee}).
\end{itemize}

To further validate the above theoretical results and the effectiveness of our proposed experience replay method, we conduct an empirical evaluation consisting of real-world and synthetic datasets. The results align with our theoretical findings and demonstrate the effectiveness of our proposed approach. Our novel experience replay method adeptly addresses both the structural shift and catastrophic forgetting, outperforming contemporary experience replay methods in NGCL \citep{cnc_er, ahrabian2021structure, kim2022dygrain}.

\begin{figure*}[!t]
\centering
\vspace{-5mm}
\includegraphics[width=0.95\textwidth]{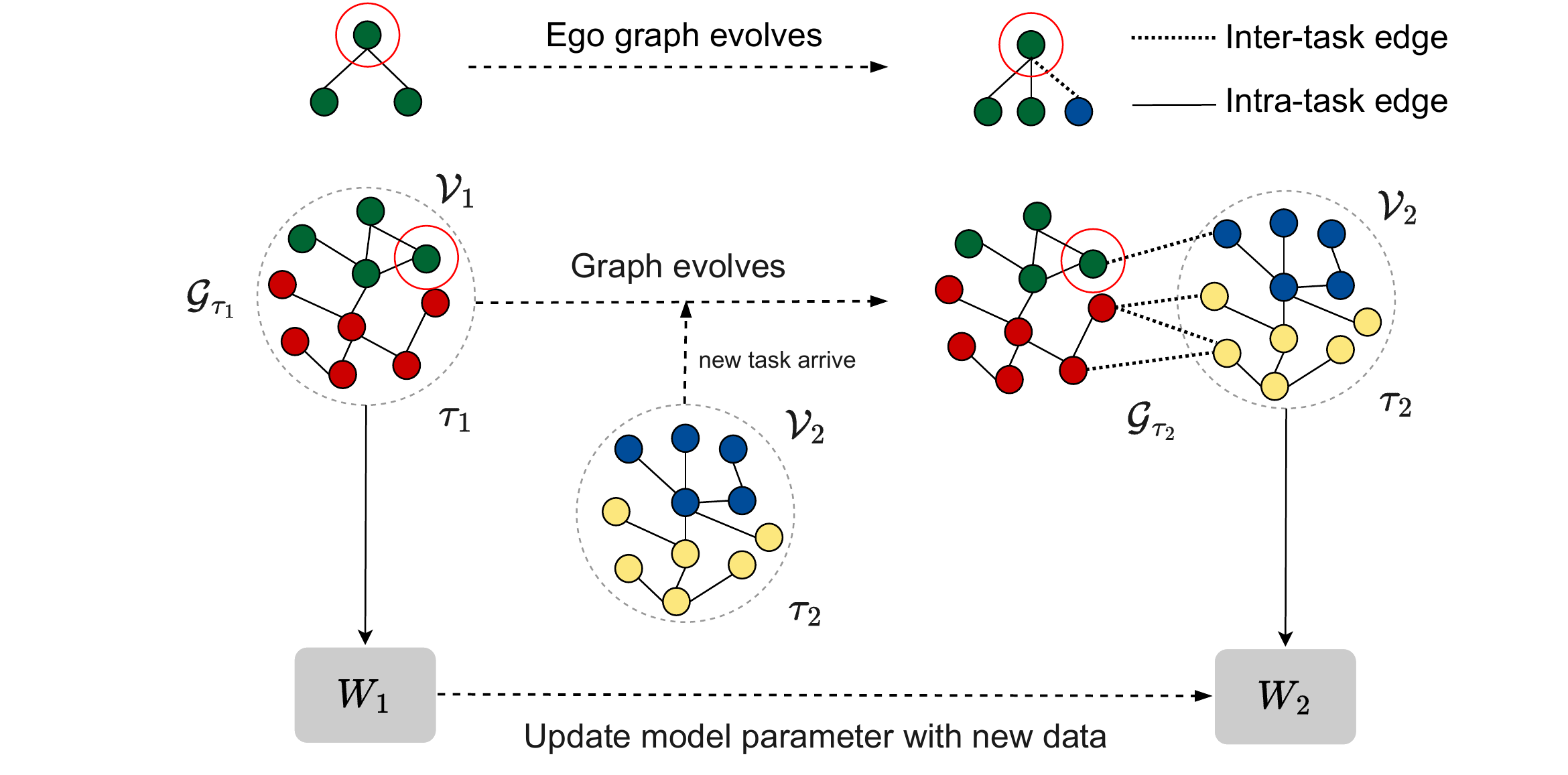}
\caption{Illustration of the progression of NGCL. Task $2$ introduces a new batch of vertices and results in an update to the parameters of the model from $W_{2}$ to $W_{1}$ using data from the new task. As new vertex batches associated with the new task $\tau_2$ are introduced, the graph structure changes, potentially altering the graph structure (inputs to GNNs) of the existing vertices, as captured by the changes in their ego graphs. This is referred to as the structural shift.}
\label{fig:NGCL}
\end{figure*}


\section{Related Work}\label{sec:related_work}
\subsection{Continual Learning}
Continual learning, also known as incremental or lifelong learning, has gained increasing attention in recent years and has been extensively explored on Euclidean data. We refer readers to surveys \citep{cl_survey,cl_survey2,cl_nlp} for a more comprehensive review of these works. The primary challenge of continual learning is to address the catastrophic forgetting problem, which refers to the degradation in the performance of a model on previous tasks after being trained on new tasks.

Existing approaches for addressing this problem can be broadly categorized into three types: regularization-based methods, experience-replay-based methods, and parameter-isolation-based methods. Regularization-based methods aim to maintain the model's performance on previous tasks by penalizing large changes in the model parameters \citep{jung2016less,li2017learning,kirkpatrick2017overcoming,farajtabar2020orthogonal,saha2021gradient}. Parameter-isolation-based methods prevent drastic changes to the parameters that are important for previous tasks by continually introducing new parameters for new tasks \citep{rusu2016progressive,yoon2017lifelong,yoon2019scalable,wortsman2020supermasks,wu2019large}. Experience-replay-based methods, inspired by the spaced repetition in human learning theory, can be viewed as a form of memory consolidation for neural networks \citep{gcl_er}. Experience replay methods consist of two key components: 1) experience selection and 2) replay strategy. The goal of experience selection is to select a small set of representative/strong samples from previous data and store them in an experience buffer. Then, the replay strategy is to strategically retrain the model on the data from the experience buffer to consolidate past knowledge and counteract catastrophic forgetting.

In this study, we delve into both the theoretical foundations and practical implementations of GNN in NGCL. On the theoretical front, we leverage established statistical learning frameworks previously utilized in domains such as domain adaptation~\citep{impossibility_expressive_divergence}, multi-task learning~\citep{ben2008notion}, and continual learning~\citep{benavides2022theory} at large. Specifically, we employ these frameworks to delve into the nature of structural shifts and their implications for the learnability of GNN within NGCL contexts.
Methodologically, we focus on the experience replay method for its simplicity and effectiveness \citep{lopez2017gradient,shin2017continual,aljundi2019gradient,caccia2020online,chrysakis2020online,knoblauch2020optimal,su2024pres}. Because of the non-I.I.D. nature of graph-structured data, experience replay for NGCL requires different treatment~\citep{gcl_er}. In this work, we propose a novel experience replay method tailored for counteracting the catastrophic forgetting and structural shift for GNNs in NGCL.

\subsection{Graph Continual Learning (GCL)}
Recently, there has been a surge of interest in GCL due to its practical significance in various applications \citep{wang2022lifelong,xu2020graphsail,daruna2021continual,kou2020disentangle,ahrabian2021structure,cai2022multimodal,wang2020bridging,liu2021overcoming,zhang2021hierarchical,zhou2021overcoming,carta2021catastrophic,zhang2022cglb,kim2022dygrain,tan2022graph}.
For a detailed review of GCL methodologies and their efficacy, we direct readers to recent reviews and benchmarks \citep{zhang2022cglb,yuan2023continual,febrinanto2023graph}. However, most existing NGCL studies primarily focus on scenarios where the entire graph structure is predefined or where interconnections between tasks are disregarded. The NGCL setting, where the graph evolves with the introduction of new tasks, remains relatively unexplored and warrants a deeper theoretical investigation. To our knowledge, \citep{su2023towards} is the only existing theoretical study on NGCL with an evolving graph. \cite{su2023towards} prove that the catastrophic forgetting risk (upper bound) of GNNs is determined by the structural shift of the underlying graph structure. In contrast, our work addresses the learnability of GNNs in NGCL under structural shift, which can be interpreted as a sort of ``lower bound''. We present the first impossibility result regarding the learnability of GNNs in NGCL, underscoring the critical role of structural shifts. Additionally, we introduce a novel experience replay method that effectively tackles both catastrophic forgetting and structural shifts in NGCL.

\subsection{Other Related Works}
While our primary focus is on GCL, it is worth noting there are parallel lines of inquiry within dynamic graph learning and graph unlearning that intersect with our work in terms of context and terminology. Dynamic graph learning focuses on enabling GNNs to capture the changing graph structures \citep{galke2021lifelong,wang2020streaming,han2020graph,yu2018netwalk,nguyen2018continuous,ma2020streaming,feng2020incremental,bielak2022fildne,su2024pres}. The goal of dynamic graph learning is to capture the temporal dynamics of the graph into the representation vectors while having access to all previous information. In contrast, GCL addresses the problem of catastrophic forgetting, in which the model's performance on previous tasks degrades after learning new tasks. For evaluation, a dynamic graph learning algorithm is only tested on the latest data, while GCL models are also evaluated on past data. On the other hand, the goal of graph unlearning is to "unlearn" or "forget" specific patterns, associations, or information that the GNN has acquired during training on a given dataset or set of tasks. Therefore, dynamic graph learning, graph unlearning, and GCL are independent and distinct research directions with different focuses and should be considered separately.

\section{Learnability of GNN in NGCL}\label{sec:learnability}

\subsection{Preliminary of NGCL}\label{ssec:preliminary}
We use lowercase letters to denote scalars and graph-related objectives. Furthermore, we use lower and uppercase boldface letters to denote vectors and matrices, respectively. NGCL assumes the existence of a stream of training tasks $\task_1, \task_2,..., \task_m$, characterized by observed vertex batches $\vertexSet_1,\vertexSet_2,...,\vertexSet_m$. Each vertex $v$ is associated with a node feature vector $\xb_v \in \mathcal{X}$ and a target label $y_v \in \mathcal{Y}$. The observed graph structure at training task $\task_i$ is induced by the accumulative vertices and given by $\graphStruct_{\task_i} = \graphStruct[\bigcup_{j=1}^i \vertexSet_j]$. In this setting, the graph structure is evolving as the learning progresses through different training tasks, as illustrated in Figure.~\ref{fig:NGCL}.

Node-level learning tasks with GNNs rely on information aggregation within the $k$-hop neighbourhood of a node as input. To accommodate this nature, we adopt a local view in the learning problem formulation. We denote $\neighbor{k}(v)$ as the k-hop neighborhood of vertex $v$, and the nodes in $\neighbor{k}(v)$ form an ego-graph $\ego{v}$, which consists of a (local) node feature matrix $\Xb_v = \{x_u |u \in \neighbor{k}(v)\}$ and a (local) adjacency matrix $\Ab_v = \{a_{uw}|u,w \in \neighbor{k}(v)\}$. We denote $\mathbb{G}$ as the space of possible ego-graph and $\Gb_{v} = (\Ab_v,\Xb_v)$ as the ego-graph for the target vertex $v$.  Let $\ego{\vertexSet} = \{ \ego{v} |v \in \vertexSet \}$ denote the set of ego graphs associated with vertex set $\vertexSet$. From the perspective of GNNs, the ego-graph $\ego{v}$ is the Markov blanket containing all necessary information for the prediction problem for the root vertex $v$. Therefore, we can see the prediction problem associated with data $\{(\ego{v},y_v)\}_{v \in \vertexSet_i}$ from training session $\task_i$ as drawn from a joint distribution $\prob(\mathbf{y_{v}},\ego{v}|\vertexSet_i)$. The dependency of the vertex is captured and manifested through the underlying graph structure in the above formulation.

Let $\hypothesis$ denote the hypothesis space and $F \in \hypothesis$ be a classifier with $\hat{y}_v = F (\ego{v})$ and $\loss(.,.) \mapsto \R$ be a given loss function. We use $R_{\jointDistCondVertSet{v}{\vertexSet}}^{\loss}(F)$ to denote the generalization performance of the classifier $f$ with respect to $\loss$ and $\jointDistCondVertSet{v}{\vertexSet}$, and it is defined as follows:
 \begin{equation}\label{eq:gen_risk}
     R_{\jointDistCondVertSet{v}{\vertexSet}}^{\loss}(F) = \expect_{\jointDistCondVertSet{v}{\vertexSet}}[\loss(F(\ego{v}),y_v)].
 \end{equation}

{\bf Catastrophic Forgetting.}
With the formulation above, the catastrophic forgetting ($\cataForget$) of a classifier $F$ after being trained on $\task_m$ can be characterized by the retention of performance on previous vertices/tasks, given by: 
\begin{equation}\label{eq:cf}
\begin{split}
      \cataForget(F) & := R_{\jointDistCondVertSet{v}{\vertexSet_1,...,\vertexSet_{m-1}}}^{\loss}(F),
\end{split}
\end{equation}
which translates to the retention of performance of the classifier $F$ from $\task_m$ on the previous tasks $\task_1,...,\task_{m-1}$.

\subsection{Structural Shift and Learnability}
This paper delves into the relationship between structural shifts and the learnability of GNNs in NGCL. We begin with defining the expressiveness of the hypothesis space and characterizing the structural shift.

\begin{definition}[Expressiveness of hypothesis space]\label{def:hypothesis_expressive}
For a set of distributions $\{\prob_1,\prob_2,...,\prob_k\}$ and a hypothesis space $\hypothesis$, we define the joint prediction error of the hypothesis space on the given distributions as, $$\lambda_{\hypothesis}(\{\prob_1,\prob_2,...,\prob_k\}) = \inf_{F \in \hypothesis}\left[ \sum_{\prob \in \{\prob_1,\prob_2,...,\prob_k\}}R_{\prob}^{\mathcal{L}}(F) \right].$$
\end{definition}

The joint prediction error characterizes the capability (expressiveness) of the hypothesis space to simultaneously fit a given set of distributions. Clearly, the expressiveness of the hypothesis space is a trivial necessary condition for the learnability of a given setting. If the expressiveness of the hypothesis space is insufficient, even the best classifier within the hypothesis space would yield poor performance. Throughout the paper, we focus our analysis on structural shifts, assuming that the training samples for each task and the expressive power of the hypothesis space are sufficient. This means that when all the tasks are presented together, there exists a learning algorithm that can produce a classifier within the hypothesis space that incurs minimal error. Further discussion on expressiveness and sample complexity is provided in Appendix~\ref{appendix:connection_proof}.

Next, we formalize the notion of structural shift among different vertex batches (tasks).

 \begin{definition}
Let $\prob_1$ and $\prob_2$ be two distributions over some domain $\mathcal{X}$, $\hypothesis\Delta \hypothesis$-distance between $\prob_1$ and $\prob_2$ is defined as, $$d_{\hypothesis\Delta \hypothesis}(\prob_1,\prob_2) = 2 \sup_{A \in \hypothesis\Delta \hypothesis}\|\prob_1(A) - \prob_2(A)\|,$$ where $\hypothesis\Delta \hypothesis = \{ \{ \xb \in \mathcal{X}:F(\xb) \neq F'(\xb) \}: F, F' \in \hypothesis\}$.
 \end{definition}
 
$\hypothesis\Delta \hypothesis$-distance is commonly used to capture the relatedness (distance) of distributions with respect to the hypothesis space $\hypothesis$~\citep{da_survey}. Here, we use it to characterize the effect of the emerging task on the distribution of the previous vertex tasks. If the structural shift among different tasks is large, then the appearance of the new vertex batch would shift the distribution of the previous vertex batch to a larger extent. Consequently, the $\hypothesis\Delta \hypothesis$-distance between the updated distribution and the previous distribution would be large

Without loss of generality, we may gauge our learnability analysis toward the case of two training sessions, referred to as the NGCL-2 problem. If the NGCL-2 problem is not learnable, then trivially, the general NGCL-k is not learnable. Furthermore, the corresponding insight can also be easily extended to multiple training tasks by recursively applying the analysis and treating consecutive tasks as a joint task. The NGCL-2 problem is characterized by three distributions: $\prob(y, \ego{v}|\vertexSet_1, \graphStruct_{\task_1})$, $\prob(y, \ego{v}|\vertexSet_1, \graphStruct_{\task_2})$, and $\prob(y, \ego{v}|\vertexSet_2, \graphStruct_{\task_2})$, which are the distributions of $\vertexSet_1$ in graphs $\graphStruct_{\task_1}$ and $\graphStruct_{\task_2}$, and the distribution of $\vertexSet_2$ in graph $\graphStruct_{\task_2}$. Then, the following theorem captures the learnability of GNN in NGCL-2.

\begin{theorem}[necessity of controllable structural shift, informal]\label{thm:divergence_necessity}
Let $\hypothesis$ be a hypothesis space over $\mathbb{G} \times \mathcal{Y}$ captured by GNNs. If there is no control on
$$d_{\hypothesis\Delta \hypothesis}\big (\prob(y, \ego{v}|\vertexSet_1, \graphStruct_{\task_1}),\prob(y, \ego{v}|\vertexSet_2, \graphStruct_{\task_2})\big ) \quad \mathrm{and} \quad d_{\hypothesis\Delta \hypothesis}\big (\prob(y, \ego{v}|\vertexSet_1, \graphStruct_{\task_1}),\prob(y, \ego{v}|{\color{blue}\vertexSet_1}, \graphStruct_{\task_2})\big )$$

then, for every $c > 0$, there exists an instance of NGCL-2 problem satisfying the following:
\begin{enumerate}
    \item $\lambda_{\hypothesis} (\{\prob(y, \ego{v}|\vertexSet_1, \graphStruct_{\task_1}), \prob(y, \ego{v}|\vertexSet_1, \graphStruct_{\task_2})\}) \leq c,$
    \item for any given continual learning algorithms, there is at least $1/2$ probability that it would produce a classifier $F$ from $\hypothesis$ with  $\cataForget(F)$ at least $1/2$.
\end{enumerate}
\end{theorem}

The formal version and proof of Theorem~\ref{thm:divergence_necessity} can be found in Appendix~\ref{appendix:connection_proof}. Theorem~\ref{thm:divergence_necessity} indicates that even if the hypothesis space is expressive enough (i.e., each individual task and their combination are learnable by the GNNs), the NGCL problem might still not be learnable without control on the structural shift among different vertex batches. We make the following remarks.
\begin{remark}\label{rem:model_selection}
    Theorem~\ref{thm:divergence_necessity} suggests that when considering continual learning for a volatile network (a large structural shift), a GNN-like model is not a suitable candidate, as its learnability is heavily damaged by the structural shift. Different treatment is required. This renders an interesting future direction but beyond the scope of this paper. 
\end{remark}

\begin{remark}\label{rem:structural_shift}
    The structural shift 
    can harm the learnability of GNNs in NGCL. Recognizing this pivotal role of structural shift, we conclude that it is important to control structural shifts even in scenarios where GNNs are learnable in NGCL. This motivates our design of the experience replay method described in the next section and can be extended to modify other continual learning methods for GNNs in NGCL.
\end{remark}

\section{Structure-Evolution-Aware Experience Replay (SEA-ER)}\label{sec:cl_framework}

In the preceding sections, we explored the learnability of GNN within the NGCL context, highlighting the significance of structural shifts in this learnability. Shifting focus from theoretical aspects to practical implementations in this section, we examine the application of the experience replay method in NGCL. We introduce SEA-ER, which integrates the topology awareness of GNNs with structural alignment to mitigate catastrophic forgetting and address structural shifts in the NGCL problem. Experience replay involves two key components: 1) experience buffer selection, which entails choosing representative data samples from the past, and 2) replay strategy, determining how the experience buffer is utilized in training with new tasks. The pseudo-code for our methods can be found in Appendix~\ref{appendix:framework}.

\subsection{Topology-Aware Experience Buffer Selection}
We denote the experience buffer for task $\task_i$ as $\experienceBuffer_i = \{P_1, P_2, ..., P_{i-1}\}$, where $P_j$ represents samples selected from task $\task_j$. GNNs have demonstrated superior generalization performance on vertices with closely related structural characteristics~\citep{subgroup}. Accordingly, we propose an experience buffer selection strategy that chooses samples with the closest structural similarity to the remaining vertices. This strategy is formulated as the optimization problem:
\begin{equation}\label{eq:experience_buff}
\begin{split}
       &\argmin_{P_j \subset \mathcal{V}_j} \max_{u \in \mathcal{V}_j \backslash P_j} \min_{v \in P_j} d_{\mathrm{spd}}(u,v), \\
        & s.t. \quad |P_j|  = b.
\end{split}
\end{equation}
Here, $d_{\mathrm{spd}}(.)$ represents the shortest path distance in the graph, and $b$ is the number of samples (budget) to select from each task. The optimization above returns a set $P_j$  with close structural similarity to the rest of the vertices within the same task. For our experiments, we adopt the classic farthest-first traversal algorithm to obtain a solution.

\subsection{Replay Strategy with Structural Alignment}
The standard replay strategy incorporates the experience buffer into the training set of the newest task, treating the samples in the buffer as normal training data. However, structural shifts induced by evolving graph structures from $\graphStruct_{\task_{i-1}}$ to $\graphStruct_{\task_{i}}$ can alter the distribution of vertices from the previous task, potentially leading to suboptimal learning. To mitigate the impact of structural shift, we propose a replay strategy with structural alignment. This strategy assigns higher weights to replay samples that are more similar to the distribution before the structural shift, ensuring that the model prioritizes relevant information. The learning objective with structural alignment is given by:
\begin{equation}\label{eq:re_learn}
    \loss = \frac{1}{|\vertexSet^{\mathrm{train}}_i|} \sum_{v \in \vertexSet^{\mathrm{train}}_i} \loss(v) + \sum_{P_j \in \experienceBuffer_i}\frac{1}{|P_j|} \sum_{v \in P_j}{\color{blue} \beta_v} \loss(v),
\end{equation}
where $\beta_v$ (highlighted in blue) is the main difference from the standard replay strategy and represents the weight for vertex $v$ in the experience buffer. To determine $\beta$ (collection of $\beta_v$), we use kernel mean matching (KMM)~\cite{gretton2009covariate} with respect to the embeddings of vertices in the experience buffer. This involves solving the following convex quadratic problem:
\begin{equation}\label{eq:kmm_weight}
\begin{split}
    & \min_{\beta} \Big\| \sum_{P_j \in \experienceBuffer_i}\sum_{v \in P_j} \beta_v \phi(h_v) -\phi(h_v')\Big\|^2, \quad \\
    & s.t. \quad B_l \leq \beta_v < B_u,
\end{split}
\end{equation}
\noindent where $h_v$ is the node embedding of $v$ (output of the GNN) with respect to $\graphStruct_{\task_i}$ and $h_v'$ is the node embedding of $v$ with respect to $\graphStruct_{\task_{i-1}}$. Eq.~\ref{eq:kmm_weight} is equivalent to match the mean elements in some kernel space $\mathcal{K}(.,.)$ and $\phi(.)$ is the feature map to the reproducing kernel Hilbert space induced by the kernel $\mathcal{K}(.,.)$. In our experiment, we use a mixture of Gaussian kernel  $\mathcal{K}(x,y) = \sum_{\alpha_i} e^{-\alpha_i ||x-y||_2}$, $\alpha_i = 1, 0.1, 0.01$. The lower bound $B_l$ and upper bound $B_u$ constraints are to ensure reasonable weight for most of the instances.

\subsection{Theoretical Analysis of SEA-ER}\label{sec:theory}
In this section, we provide a theoretical foundation for our proposed method by introducing the concept of distortion rate, which helps characterize the topological awareness of GNNs.

\begin{definition}[distortion rate]\label{def:distortion}
    Given two metric spaces $(\mathcal{Q},d)$ and $(\mathcal{Q'},d')$ and  
 	a mapping $\gamma: \mathcal{Q} \mapsto \mathcal{Q'}$, $\gamma$ is said to have a distortion $\alpha \geq 1$, if there exists a constant $r > 0$ such that $\forall u,v \in \mathcal{E}$, $$ r d(u,v) \leq d'(\gamma(u),\gamma(v))\leq \alpha r d(u,v).$$
\end{definition}

 The distortion rate quantifies the ratio of distances in the embedding space to those in the original graph space. A low distortion rate, close to 1, signifies that the embedding effectively preserves the structural information from the original graph. This measure allows us to assess the topological awareness of GNNs and forms the basis for our theoretical analysis of the experience selection strategy in relation to the underlying graph topology.

\begin{proposition}\label{prop:performance_guarantee}
Given a collection task $\task_1,...,\task_m$ and their associated vertex set $\vertexSet_1,...,\vertexSet_m$. Let $\experienceBuffer_m = \{P_1,...,P_{m-1}\}$ be the experience buffer where each $P_i \in \experienceBuffer_m$ has the same cardinality and cover all classes in $\task_i$. 
Let $F$ be the learned GNN model with distortion rate $\alpha$. Assuming the loss function $\loss$ is bounded, continuous and smooth, we have that,
    $$ R_{\jointDistCondVertSet{v}{\vertexSet_i \backslash P_i}}^{\loss}(F) \leq R_{\jointDistCondVertSet{v}{ P_i}}^{\loss}(F) + \mathcal{O} (\alpha  D_s(\vertexSet_i \backslash P_i,P_i)),$$
where 
$$
  D_s(\vertexSet_i \backslash P_i,P_i) = \max_{u \in \vertexSet_i \backslash P_i}\min_{v \in P_i} d_{\mathrm{spd}}(v, u).
$$
\end{proposition}

The proof for Propostion~\ref{prop:performance_guarantee} can be found in Appendix~\ref{appendix:structural_relation}. This proposition demonstrates that catastrophic forgetting in the model, as measured by its ability to retain performance on past data, is bounded by the training loss on the experience buffer and is further influenced by a combination of factors involving the distortion rate of the model and the structural distance within the experience buffer. This underscores the critical role of graph structure when selecting a suitable experience buffer, as the structural distance substantially impacts the performance bound. Consequently, Proposition~\ref{prop:performance_guarantee} validates the optimization objective, Eq.~\ref{eq:experience_buff}, for selecting an experience buffer and provides a performance guarantee for our proposed experience selection strategy.

\begin{figure}[t!]
\subfigure[Evolution of Task 1 Performance]{
\includegraphics[width=.35\textwidth]{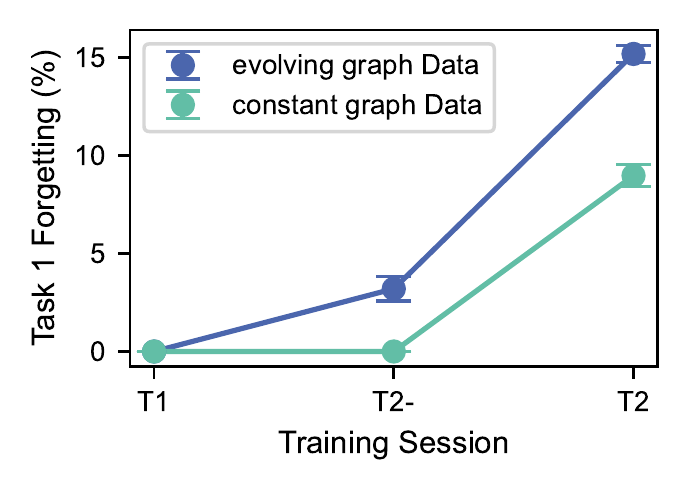}
\label{fig:structural_dependency}
}
\subfigure[Arxiv, Constant]{
\includegraphics[width=.285\textwidth]{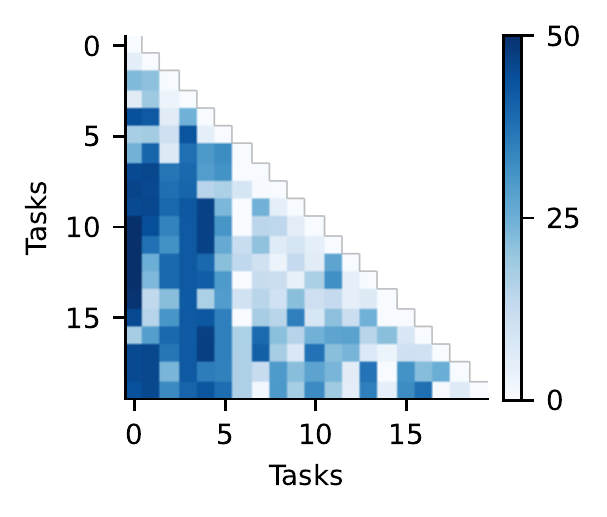}
\label{fig:forget_dynamic1}
}\hspace{-4mm}
\subfigure[Arxiv, Evolving]{
\includegraphics[width=.285\textwidth]{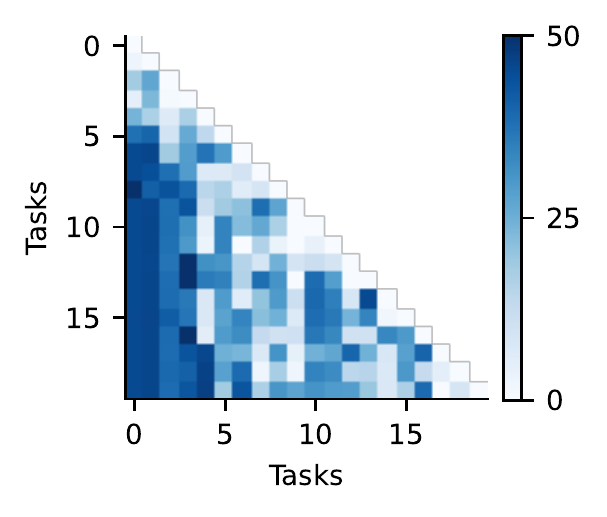}
\label{fig:forget_dynamic2}
}
\caption{Forgetting Dynamics of Bare Model on Arxiv under Settings with Constant and Evolving Graphs. (a) captures the change of catastrophic forgetting of task 1 when transitioning into task 2 under the settings with constant and evolving graphs. (b) and (c) are the complete catastrophic forgetting matrix ($x,y$-axis are the $i,j$ in $r_{i,j}-r_{i,i}$ correspondingly) of Bare model in the inductive and transductive settings. The colour density indicates the amount of forgetting(denser color means larger forgetting).}
\label{fig:diff}
\end{figure}

\begin{table*}[t!]
  \centering
  \caption{Performance comparison of SEA-ER to existing experience replay NGCL frameworks ($\uparrow$ higher means better). Results are averaged among five trials. {\bf Bold letter} indicates that the entry admits the best performance for that given dataset. $5\%$ of the training set size is used for the size of the experience replay buffer for all the experience replay methods. 
  }  
  \resizebox{\textwidth}{!}{
    \begin{tabular}{|c|cc|cc|cc|}
    \hline
        Dataset  &        \multicolumn{2}{c}{CoraFull}         & \multicolumn{2}{c}{OGB-Arxiv} & \multicolumn{2}{c}{Reddit} \\
    \hline
        Performance Metric& FAP (\%) $\uparrow$ & FAF (\%) $\uparrow$ & FAP (\%) $\uparrow$ & FAF (\%) $\uparrow$ & FAP(\%) $\uparrow$ & FAF (\%) $\uparrow$\\
    \hline
    Bare model & 61.41 $\pm$ 1.2 & -30.36 $\pm$ 1.3 & 59.24 $\pm$ 3.6 & -33.72  $\pm$ 3.3 & 68.60 $\pm$ 1.8 & -23.87 $\pm$ 1.7 \\
    Joint Training & 92.34 $\pm$ 0.6 & N.A.  & 94.36 $\pm$ 0.4 & N.A.   & 95.27 $\pm$ 1.2 & N.A. \\
    \hline
        ER:IID-Random & 77.24 $\pm$ 5.8 & -16.73 $\pm$ 5.6  & 79.72 $\pm$ 4.6 & -14.89 $\pm$ 4.5  & 80.08 $\pm$ 3.3 & -10.89 $\pm$ 2.5 \\
        ER-deg & 86.24 $\pm$ 0.7 & -8.73 $\pm$ 0.7 & 85.28 $\pm$ 0.8 & -10.89 $\pm$ 1.9  & 85.08 $\pm$ 0.9 & -8.89 $\pm$ 0.9\\
        ER-infl & 87.99 $\pm$ 0.6 & -8.31 $\pm$ 0.5 & 86.42 $\pm$ 1.2 & -7.65  $\pm$ 1.3 & 86.98 $\pm$ 1.4 & -6.82 $\pm$ 1.6\\
        ER-rep  & 88.95 $\pm$ 0.9 & -7.55 $\pm$ 0.8 & 88.12 $\pm$ 0.9 & -8.23  $\pm$ 0.9 & 86.02 $\pm$ 1.2 & -7.21 $\pm$ 1.3\\
        SEA-ER w.o. sa (ours) & 89.13 $\pm$ 0.8 & -7.21 $\pm$0.9 & 89.19 $\pm$ 0.7  & -7.13 $\pm$ 0.5 & 88.48 $\pm$ 1.7 & -6.10 $\pm$ 1.6\\
    \hline
          ER:IID-Random w. sa & 80.24 $\pm$ 6.9 & -14.73 $\pm$ 4.8  & 81.72 $\pm$ 5.2 & -13.23 $\pm$ 5.2  & 83.43 $\pm$ 2.8 & -10.32 $\pm$ 3.8 \\
        ER-deg w. sa & 87.32 $\pm$ 0.5 & -7.89 $\pm$ 0.5 & 86.88 $\pm$ 0.6 & -9.32 $\pm$ 2.0  & 87.11 $\pm$ 1.0 & -6.89 $\pm$ 1.0\\
        ER-infl w. sa & 88.88 $\pm$ 0.5 & -8.01 $\pm$ 0.5 & 89.42 $\pm$ 1.3 & -5.65  $\pm$ 1.7 & 90.64 $\pm$ 1.6 & -3.63 $\pm$ 1.3\\
        ER-rep w. sa  & 89.06 $\pm$ 0.7 & -7.00 $\pm$ 0.9 & 88.92 $\pm$ 0.8 & -8.00  $\pm$ 1.1 & 88.99 $\pm$ 1.4 & -4.21 $\pm$ 1.1\\
        SEA-ER (ours) &  {\bf 91.67} $\pm$ 1.3 & {\bf -4.01} $\pm$ 1.4& {\bf 92.88} $\pm$ 1.1 & {\bf -3.08} $\pm$ 0.7   & {\bf 92.89} $\pm$ 1.5 &  {\bf -2.72} $\pm$  1.4\\
    \hline
    \end{tabular}%
   }
  \label{tab:replay}%
\end{table*}%

\begin{table}[t!]
  \centering
\resizebox{0.6\textwidth}{!}{
    \begin{tabular}{c|ccc}
    \hline
        Datasets  &        OGB-Arxiv  & Reddit & CoraFull          \\   
    \hline
           \# vertices & 169,343 & 227,853 & 19,793\\
           \# edges   & 1,166,243 & 114,615,892 & 130,622 \\
           \# class & 40 & 40 & 70\\
    \hline
         \# tasks & 20 & 20 & 35\\
         \# vertices / \# task & 8,467 & 11,393 & 660\\
         \# edges / \# task & 58,312 & 5,730,794 & 4,354\\
   \hline
    \end{tabular}
    }
    \captionof{table}{Continual learning settings for each dataset.}
  \label{tab:data_description}
\end{table}


\section{Experiments}\label{sec:evaluation}
In this section, we present an empirical validation of our theoretical results and evaluation of our proposed experience replay method for NGCL. In this empirical study, we aim to answer the following questions: (1) Does structural shift indeed have an impact on the learnability and catastrophic forgetting of GNNs in NGCL? (2) Does the distortion rate of GNNs, as discussed in Proposition~\ref{prop:performance_guarantee}, hold true in practice? (3) How does our proposed experience replay method compare to existing experience replay methods for GNNs~\citep{cnc_er,ahrabian2021structure,kim2022dygrain} in NGCL? Due to space limitations, we only summarize the key details of the experiments in the main paper and defer a more comprehensive description of the datasets, experiment set-up and additional results in Appendix~\ref{appendix:exp}.

\paragraph{  Datasets and General Set-up.} In general, we follow closely the experiment settings in ~\citep{su2023towards,zhang2022cglb}. We conduct experiments on both real-world and synthetic datasets. Real-world datasets include OGB-Arxiv~\citep{ogb}, Reddit~\citep{reddit}, and CoraFull~\citep{cora_full}. To generate the synthetic dataset, we use the popular contextual stochastic block model (cSBM)~\citep{cSBM} (further details of the cSBM can be found in Appendix~\ref{appendix:sbm}). The experimental set-up follows the widely adopted task-continual-learning (task-CL)\citep{su2023towards,zhang2022cglb}, where a k-class classification task is extracted from the dataset for each training session. For example, in the OGB-Arxiv dataset, which has 40 classes, we divide them into 20 tasks: Task 1 is a 2-class classification task between classes 0 and 1, task 2 is between classes 2 and 3, and so on. In each task, the system only has access to the graph induced by the vertices at the current and earlier learning stages, following the formulation in Sec.\ref{ssec:preliminary}. A brief description of the datasets and how they are divided into different node classification tasks is given in Table~\ref{tab:data_description}. We adopt the implementation from the recent NGCL benchmark~\citep{zhang2022cglb} for creating the task-CL setting and closely following their set-up (such as the train/valid/test set split).

\paragraph{  Evaluation Metrics.} Let $r_{i,j}$ denote the performance (accuracy) on task $j$ after the model has been trained over a sequence of tasks from $1$ to $i$. Then, the forgetting of task $j$ after being trained over a sequence of tasks from $1$ to $i$ is measured by $r_{i,j}-r_{j,j}$.  We use the final average performance (FAP) $:= \sum_{j}^m r_{m,j}/{m}$ and the final average forgetting (FAF) $:= (\sum_{j}^m r_{m,j}-r_{j,j})/{m}$ to measure the overall effectiveness of an NGCL framework. These metrics measure the performance of the model after being presented with all the tasks. The order of the task permutation should not affect the performance of the experience reply with the two chosen metrics (further result and discussion on this regard can be found in Appendix~\ref{appendix:exp}

\paragraph{  Baselines.} We compare the performance of SEA-ER with the following experience replay baselines. All the baselines are defaulted for the standard replay strategy, i.e. Eq.~\ref{eq:re_learn} without $\beta$.  {\bf ER-IID-Random} updates each task with an experience buffer with vertices randomly selected from each task. {\bf ER-IID-Random} serves as a baseline for showing the effectiveness of different experience buffer selection strategies. We adopt state-of-the-art experience replay frameworks for GNNs for comparison. 
{\bf ER-rep}~\citep{cnc_er} selects an experience buffer based on the representations of the vertex. {\bf ER-deg}~\citep{ahrabian2021structure} selects an experience buffer selected based on the degree of vertices. {\bf ER-infl}~\citep{kim2022dygrain}  selects an experience buffer selected based on the influence score of the representations of the vertex. We adopt the implementation of the two baselines~\citep{ahrabian2021structure,cnc_er} from the benchmark paper~\citep{zhang2022cglb} and implement a version of \citep{dynamic_learning} based on our understanding, as the code is not open-sourced. {\bf SEA-ER w.o. sa} is our proposed experience replay selection strategy with the standard replay strategy, and {\bf SEA-ER} is our proposed experience replay selection strategy with learning objective Eq.~\ref{eq:re_learn}. The suffix {\bf w. sa} indicate the usage of Eq.~\ref{eq:re_learn}. In addition to the above frameworks, we also include two natural baselines for NGCL: the {\bf Bare model} and {\bf Joint Training}. The {\bf Bare model} denotes the backbone GNN without any continual learning techniques, serving as the lower bound on continual learning performance. On the other hand, {\bf Joint Training} trains the backbone GNN on all tasks simultaneously, resulting in no forgetting problems and serving as the upper bound for continual learning performance.

\begin{figure}[t!]
  \hfill
\subfigure[FAF of cSBM]{  
   \includegraphics[width=.315\textwidth]{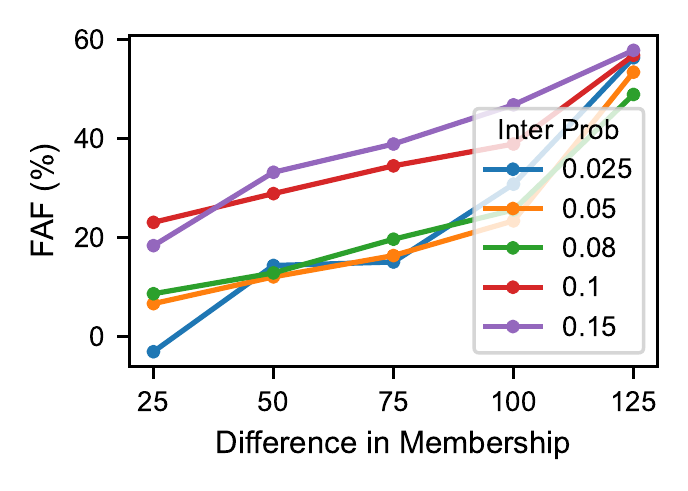}
  \label{fig:cSBM_dependency}}
  \hfill
\subfigure[Distortion]{
\includegraphics[width=.315\textwidth]{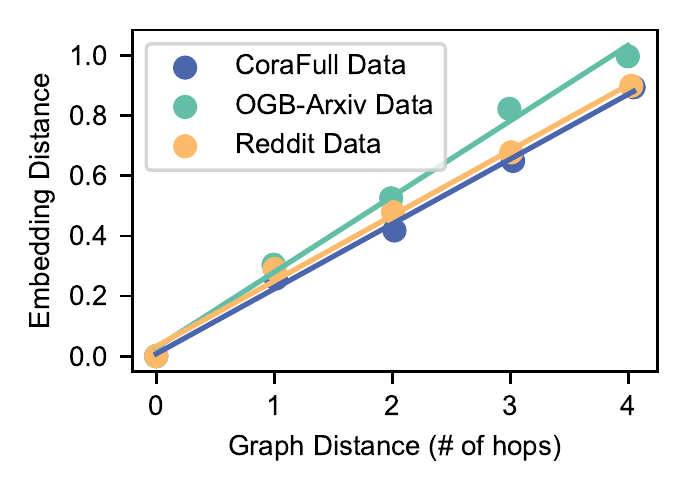}
  \label{fig:distortion}
}
\hfill
\subfigure[Performance vs.~buffer size  ]{
\includegraphics[width=.315\textwidth]{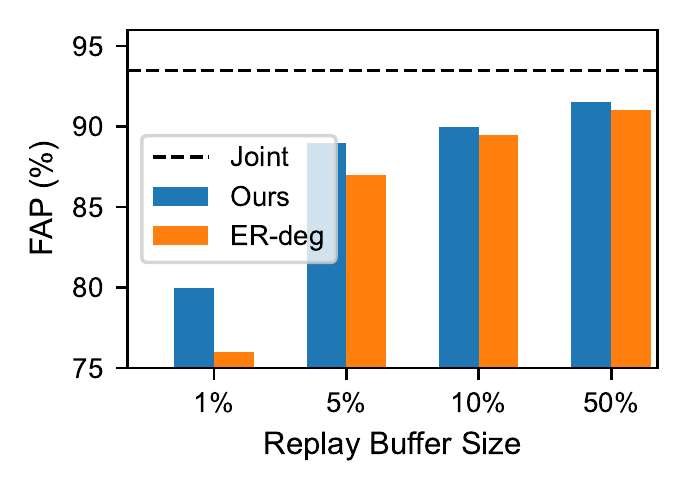}
\label{fig:buffer_size}
}
  \hfill
  \caption{Fig.~\ref{fig:cSBM_dependency} is the experiments with different configurations of cSBM models. Fig.~\ref{fig:distortion} is the distortion between graph distance and embedding distance. Fig.~\ref{fig:buffer_size} is the experiment on the effect of replay buffer size (without the structural alignment).}
\label{fig:diff}
\end{figure}

\subsection{Summary of the Empirical Results}
\paragraph{ Impact of Structural Shift.}  We first show the effect of evolving graph structure (structural shift) on the continual learning system. We do so by comparing the dynamic of forgetting (change of $r_{i,j} - r_{i,i}$ for different $i,j$) of the Bare model. The result of the OGB-Arxiv dataset is illustrated in Fig.~\ref{fig:structural_dependency}. In Fig.~\ref{fig:structural_dependency}, $T_1$ and $T_2$ on the x-axis denote the first and the second tasks (training sessions), respectively, and $T_2-$ represents the state when Task 2 has arrived but the model has not been trained on Task 2 data yet. As shown in the figure, in the case of an evolving graph, the catastrophic forgetting on Task 1 occurs twice when entering stage 2: when vertices of Task 2 are added to the graph (structural shift) and when the model is trained on Task 2's data. Fig.~\ref{fig:forget_dynamic1} and Fig.~\ref{fig:forget_dynamic2} are the complete catastrophic forgetting matrix (with $r_{i,j}-r_{i,i}$ as entries) for the two settings. The colour density difference of the two performance matrices illustrates that the existing tasks suffer more catastrophic forgetting in the setting 
 with evolving graphs. The results of this experiment align with our expectations, indicating that the existence of structural shifts can lead to a larger catastrophic forgetting.

We then further quantify the relation between the structural shift and the performance of a continual learning system. We use cSBM to create a two-session NGCL with a two-community graph of the size $300$ for each training session (task), and the learning objective is to differentiate and classify vertices between the two communities.  The community membership configuration is described by a four-tuple, $(c_0^{(1)},c_1^{(1)},c_0^{(2)},c_1^{(2)}) = \big ( 150+ \delta,150-\delta,150-\delta,150+\delta \big ),$ where $c_i^{(j)}$ is the number of vertices from community $i$ at session $j$. Then, we vary the structural shift by changing (1) the inter-connectivity $p_{\mathrm{inter}}$ (probability of an edge between tasks/sessions) and (2) the difference ($\delta$) of each community member in different stages. Intuitively, as we increase $\delta$ and the inter-connectivity $p_{\mathrm{inter}}$, the structural shift between the two training sessions increases and therefore we should observe a larger FAF converging to random guess as predicted by the theoretical results. As shown in Fig.~\ref{fig:cSBM_dependency}, this is indeed the case and therefore validating Theorem~\ref{thm:divergence_necessity}.

{\bf Distortion Rate.} We next validate that the distortion rate of GNNs is indeed small. We train a vanilla GraphSAGE model using the default setting for each dataset. We extract vertices within the 5-hop neighbourhood of the training set and group the vertices based on their distances to the training set. We then compute the average embedding distance of each group to the training set in the embedding space. In Fig.~\ref{fig:distortion}, the distortion rate and scaling factor of the GNN are reflected by the slope of the near-linear relation between the embedding distance and the graph distance. We can see that the distortion factors in the performance bound in Proposition~\ref{prop:performance_guarantee} are indeed small (closed to 1), validating our proposed experience selection strategy.

 {\bf Effectiveness of SEA-ER.} In Table ~\ref{tab:replay}, we show the FAP and FAF of each method after learning the entire task sequence. On average, the Bare model without any continual learning technique performs the worst, and Joint training performs the best. FAF is inapplicable to joint-trained models because they do not follow the continual learning setting and are simultaneously trained on all tasks. Our experience buffer selection strategy, SEA-ER w.o. sa, outperforms ER-rep, ER-infl and ER-deg in the proposed setting. In addition, the modified learning objective Eq.~\ref{eq:re_learn} with importance re-weighting can boost the performance of all experience replay frameworks by adjusting the structural shifts induced by the evolving graph structure. 

{\bf Ablation Study on the Size of Experience Replay Buffer.} As illustrated in Fig.~\ref{fig:buffer_size}, the performance of SEA-ER-re, predictably, converges to the Joint Training setting as the size of the experience replay buffer increases. Nonetheless, the size of the experience replay buffer is often subject to system constraints, such as the available storage capacity. The experiment demonstrates that SEA-ER is highly effective in preserving the performance of the system (as indicated by the final average task performance, FAP), even when the size of the experience replay buffer is limited.

\section{Conclusion}\label{sec:conclusion_discussion}
This paper offers a comprehensive theoretical exploration of the learnability of GNNs in NGCL within the context of evolving graphs. Our formal analysis reveals a crucial insight: GNNs may not be always learnable in NGCL when structural shifts are uncontrolled. Furthermore, we introduce a novel experience replay framework, SEA-ER, adept at addressing the challenges of catastrophic forgetting and structural shifts inherent in this problem. Empirical results validate our theoretical findings and underscore the effectiveness of our proposed approach.

\subsection{Limitation and Future Work}
The lower bound of catastrophic forgetting for GNNs in NGCL presented in this paper is framed with the existence of a learning algorithm. While it offers insight into the extent of structural shifts that can induce catastrophic forgetting of GNNs, establishing a more quantitative relationship between structural shifts and catastrophic forgetting of GNNs in NGCL would be a valuable avenue for further research. This could enhance our understanding of GNNs in NGCL and facilitate the development of improved methods. 

Our study here primarily investigates widely-used GNN architectures, including GCN, GAT, and SAGE. Nevertheless, in NGCL scenarios characterized by pronounced heterophily, it becomes imperative to explore heterophily-specific GNN models, which may employ distinct aggregation and update mechanisms. Such models could fundamentally alter the approach to addressing the problem. This highlights a potential avenue for significant future research within the domain of graph continual learning. However, it falls outside the purview of our current study.

\section*{Acknowledgement}
We would like to thank the anonymous reviewers and area chairs for their helpful comments. JS and CW are supported in part by Hong Kong RGC grants 17207621, 17203522 and C7004-22G (CRF). DZ is supported by NSFC 62306252 and central fund from HKU IDS.

\bibliography{reference}

\begin{thebibliography}{66}
\providecommand{\natexlab}[1]{#1}
\providecommand{\url}[1]{\texttt{#1}}
\expandafter\ifx\csname urlstyle\endcsname\relax
  \providecommand{\doi}[1]{doi: #1}\else
  \providecommand{\doi}{doi: \begingroup \urlstyle{rm}\Url}\fi

\bibitem[Ahrabian et~al.(2021)Ahrabian, Xu, Zhang, Wu, Wang, and Coates]{ahrabian2021structure}
Kian Ahrabian, Yishi Xu, Yingxue Zhang, Jiapeng Wu, Yuening Wang, and Mark Coates.
\newblock Structure aware experience replay for incremental learning in graph-based recommender systems.
\newblock In \emph{Proceedings of the 30th ACM International Conference on Information \& Knowledge Management}, pp.\  2832--2836, 2021.

\bibitem[Aljundi et~al.(2019)Aljundi, Lin, Goujaud, and Bengio]{aljundi2019gradient}
Rahaf Aljundi, Min Lin, Baptiste Goujaud, and Yoshua Bengio.
\newblock Gradient based sample selection for online continual learning.
\newblock \emph{Advances in neural information processing systems}, 32, 2019.

\bibitem[Ben-David \& Borbely(2008)Ben-David and Borbely]{ben2008notion}
Shai Ben-David and Reba~Schuller Borbely.
\newblock A notion of task relatedness yielding provable multiple-task learning guarantees.
\newblock \emph{Machine learning}, 73:\penalty0 273--287, 2008.

\bibitem[Ben-David \& Urner(2012)Ben-David and Urner]{sample_complexity}
Shai Ben-David and Ruth Urner.
\newblock On the hardness of domain adaptation and the utility of unlabeled target samples.
\newblock In \emph{International Conference on Algorithmic Learning Theory}, pp.\  139--153. Springer, 2012.

\bibitem[Benavides-Prado \& Riddle(2022)Benavides-Prado and Riddle]{benavides2022theory}
Diana Benavides-Prado and Patricia Riddle.
\newblock A theory for knowledge transfer in continual learning.
\newblock In \emph{Conference on Lifelong Learning Agents}, pp.\  647--660. PMLR, 2022.

\bibitem[Bielak et~al.(2022)Bielak, Tagowski, Falkiewicz, Kajdanowicz, and Chawla]{bielak2022fildne}
Piotr Bielak, Kamil Tagowski, Maciej Falkiewicz, Tomasz Kajdanowicz, and Nitesh~V Chawla.
\newblock Fildne: a framework for incremental learning of dynamic networks embeddings.
\newblock \emph{Knowledge-Based Systems}, 236:\penalty0 107453, 2022.

\bibitem[Biesialska et~al.(2020)Biesialska, Biesialska, and Costa-juss{\`a}]{cl_nlp}
Magdalena Biesialska, Katarzyna Biesialska, and Marta~R Costa-juss{\`a}.
\newblock Continual lifelong learning in natural language processing: A survey.
\newblock \emph{arXiv preprint arXiv:2012.09823}, 2020.

\bibitem[Bojchevski \& G{\"u}nnemann(2017)Bojchevski and G{\"u}nnemann]{cora_full}
Aleksandar Bojchevski and Stephan G{\"u}nnemann.
\newblock Deep gaussian embedding of graphs: Unsupervised inductive learning via ranking.
\newblock \emph{arXiv preprint arXiv:1707.03815}, 2017.

\bibitem[Caccia et~al.(2020)Caccia, Belilovsky, Caccia, and Pineau]{caccia2020online}
Lucas Caccia, Eugene Belilovsky, Massimo Caccia, and Joelle Pineau.
\newblock Online learned continual compression with adaptive quantization modules.
\newblock In \emph{International conference on machine learning}, pp.\  1240--1250. PMLR, 2020.

\bibitem[Cai et~al.(2022)Cai, Wang, Guan, Tang, Xu, Zhong, and Zhu]{cai2022multimodal}
Jie Cai, Xin Wang, Chaoyu Guan, Yateng Tang, Jin Xu, Bin Zhong, and Wenwu Zhu.
\newblock Multimodal continual graph learning with neural architecture search.
\newblock In \emph{Proceedings of the ACM Web Conference 2022}, pp.\  1292--1300, 2022.

\bibitem[Carta et~al.(2021)Carta, Cossu, Errica, and Bacciu]{carta2021catastrophic}
Antonio Carta, Andrea Cossu, Federico Errica, and Davide Bacciu.
\newblock Catastrophic forgetting in deep graph networks: an introductory benchmark for graph classification.
\newblock \emph{arXiv preprint arXiv:2103.11750}, 2021.

\bibitem[Chrysakis \& Moens(2020)Chrysakis and Moens]{chrysakis2020online}
Aristotelis Chrysakis and Marie-Francine Moens.
\newblock Online continual learning from imbalanced data.
\newblock In \emph{International Conference on Machine Learning}, pp.\  1952--1961. PMLR, 2020.

\bibitem[Daruna et~al.(2021)Daruna, Gupta, Sridharan, and Chernova]{daruna2021continual}
Angel Daruna, Mehul Gupta, Mohan Sridharan, and Sonia Chernova.
\newblock Continual learning of knowledge graph embeddings.
\newblock \emph{IEEE Robotics and Automation Letters}, 6\penalty0 (2):\penalty0 1128--1135, 2021.

\bibitem[David et~al.(2010)David, Lu, Luu, and P{\'a}l]{impossibility_expressive_divergence}
Shai~Ben David, Tyler Lu, Teresa Luu, and D{\'a}vid P{\'a}l.
\newblock Impossibility theorems for domain adaptation.
\newblock In \emph{Proceedings of the Thirteenth International Conference on Artificial Intelligence and Statistics}, pp.\  129--136. JMLR Workshop and Conference Proceedings, 2010.

\bibitem[Delange et~al.(2021)Delange, Aljundi, Masana, Parisot, Jia, Leonardis, Slabaugh, and Tuytelaars]{cl_survey}
Matthias Delange, Rahaf Aljundi, Marc Masana, Sarah Parisot, Xu~Jia, Ales Leonardis, Greg Slabaugh, and Tinne Tuytelaars.
\newblock A continual learning survey: Defying forgetting in classification tasks.
\newblock \emph{IEEE Transactions on Pattern Analysis and Machine Intelligence}, 2021.

\bibitem[Deshpande et~al.(2018)Deshpande, Sen, Montanari, and Mossel]{cSBM}
Yash Deshpande, Subhabrata Sen, Andrea Montanari, and Elchanan Mossel.
\newblock Contextual stochastic block models.
\newblock \emph{Advances in Neural Information Processing Systems}, 31, 2018.

\bibitem[Farajtabar et~al.(2020)Farajtabar, Azizan, Mott, and Li]{farajtabar2020orthogonal}
Mehrdad Farajtabar, Navid Azizan, Alex Mott, and Ang Li.
\newblock Orthogonal gradient descent for continual learning.
\newblock In \emph{International Conference on Artificial Intelligence and Statistics}, pp.\  3762--3773. PMLR, 2020.

\bibitem[Febrinanto et~al.(2023)Febrinanto, Xia, Moore, Thapa, and Aggarwal]{febrinanto2023graph}
Falih~Gozi Febrinanto, Feng Xia, Kristen Moore, Chandra Thapa, and Charu Aggarwal.
\newblock Graph lifelong learning: A survey.
\newblock \emph{IEEE Computational Intelligence Magazine}, 18\penalty0 (1):\penalty0 32--51, 2023.

\bibitem[Feng et~al.(2020)Feng, Jiang, and Gao]{feng2020incremental}
Yutong Feng, Jianwen Jiang, and Yue Gao.
\newblock Incremental learning on growing graphs.
\newblock \emph{openReview preprint https://openreview.net/forum?id=nySHNUlKTVw}, 2020.

\bibitem[Galke et~al.(2021)Galke, Franke, Zielke, and Scherp]{galke2021lifelong}
Lukas Galke, Benedikt Franke, Tobias Zielke, and Ansgar Scherp.
\newblock Lifelong learning of graph neural networks for open-world node classification.
\newblock In \emph{2021 International Joint Conference on Neural Networks (IJCNN)}, pp.\  1--8. IEEE, 2021.

\bibitem[Gama et~al.(2014)Gama, {\v{Z}}liobait{\.e}, Bifet, Pechenizkiy, and Bouchachia]{gama2014survey}
Jo{\~a}o Gama, Indr{\.e} {\v{Z}}liobait{\.e}, Albert Bifet, Mykola Pechenizkiy, and Abdelhamid Bouchachia.
\newblock A survey on concept drift adaptation.
\newblock \emph{ACM computing surveys (CSUR)}, 46\penalty0 (4):\penalty0 1--37, 2014.

\bibitem[Gretton et~al.(2009)Gretton, Smola, Huang, Schmittfull, Borgwardt, and Sch{\"o}lkopf]{gretton2009covariate}
Arthur Gretton, Alex Smola, Jiayuan Huang, Marcel Schmittfull, Karsten Borgwardt, and Bernhard Sch{\"o}lkopf.
\newblock Covariate shift by kernel mean matching.
\newblock \emph{Dataset shift in machine learning}, 3\penalty0 (4):\penalty0 5, 2009.

\bibitem[Hamilton et~al.(2017)Hamilton, Ying, and Leskovec]{reddit}
William~L Hamilton, Rex Ying, and Jure Leskovec.
\newblock Inductive representation learning on large graphs.
\newblock In \emph{Proceedings of the 31st International Conference on Neural Information Processing Systems}, pp.\  1025--1035, 2017.

\bibitem[Hamilton et~al.(2018)Hamilton, Ying, and Leskovec]{sage}
William~L. Hamilton, Rex Ying, and Jure Leskovec.
\newblock Inductive representation learning on large graphs, 2018.

\bibitem[Han et~al.(2020)Han, Karunasekera, and Leckie]{han2020graph}
Yi~Han, Shanika Karunasekera, and Christopher Leckie.
\newblock Graph neural networks with continual learning for fake news detection from social media.
\newblock \emph{arXiv preprint arXiv:2007.03316}, 2020.

\bibitem[Hochbaum(1996)]{approx}
Dorit~S Hochbaum.
\newblock Approximating covering and packing problems: set cover, vertex cover, independent set, and related problems.
\newblock In \emph{Approximation algorithms for NP-hard problems}, pp.\  94--143. 1996.

\bibitem[Hu et~al.(2020)Hu, Fey, Zitnik, Dong, Ren, Liu, Catasta, and Leskovec]{ogb}
Weihua Hu, Matthias Fey, Marinka Zitnik, Yuxiao Dong, Hongyu Ren, Bowen Liu, Michele Catasta, and Jure Leskovec.
\newblock Open graph benchmark: Datasets for machine learning on graphs.
\newblock \emph{arXiv preprint arXiv:2005.00687}, 2020.

\bibitem[Jung et~al.(2016)Jung, Ju, Jung, and Kim]{jung2016less}
Heechul Jung, Jeongwoo Ju, Minju Jung, and Junmo Kim.
\newblock Less-forgetting learning in deep neural networks.
\newblock \emph{arXiv preprint arXiv:1607.00122}, 2016.

\bibitem[Kazemi et~al.(2020)Kazemi, Goel, Jain, Kobyzev, Sethi, Forsyth, and Poupart]{dynamic_learning}
Seyed~Mehran Kazemi, Rishab Goel, Kshitij Jain, Ivan Kobyzev, Akshay Sethi, Peter Forsyth, and Pascal Poupart.
\newblock Representation learning for dynamic graphs: A survey.
\newblock \emph{J. Mach. Learn. Res.}, 21\penalty0 (70):\penalty0 1--73, 2020.

\bibitem[Kim et~al.(2022)Kim, Yun, and Kang]{kim2022dygrain}
Seoyoon Kim, Seongjun Yun, and Jaewoo Kang.
\newblock Dygrain: An incremental learning framework for dynamic graphs.
\newblock In \emph{31st International Joint Conference on Artificial Intelligence, IJCAI 2022}, pp.\  3157--3163. International Joint Conferences on Artificial Intelligence, 2022.

\bibitem[Kipf \& Welling(2017)Kipf and Welling]{gcn}
Thomas~N. Kipf and Max Welling.
\newblock Semi-supervised classification with graph convolutional networks, 2017.

\bibitem[Kirkpatrick et~al.(2017)Kirkpatrick, Pascanu, Rabinowitz, Veness, Desjardins, Rusu, Milan, Quan, Ramalho, Grabska-Barwinska, et~al.]{kirkpatrick2017overcoming}
James Kirkpatrick, Razvan Pascanu, Neil Rabinowitz, Joel Veness, Guillaume Desjardins, Andrei~A Rusu, Kieran Milan, John Quan, Tiago Ramalho, Agnieszka Grabska-Barwinska, et~al.
\newblock Overcoming catastrophic forgetting in neural networks.
\newblock \emph{Proceedings of the national academy of sciences}, 114\penalty0 (13):\penalty0 3521--3526, 2017.

\bibitem[Knoblauch et~al.(2020)Knoblauch, Husain, and Diethe]{knoblauch2020optimal}
Jeremias Knoblauch, Hisham Husain, and Tom Diethe.
\newblock Optimal continual learning has perfect memory and is np-hard.
\newblock In \emph{International Conference on Machine Learning}, pp.\  5327--5337. PMLR, 2020.

\bibitem[Kou et~al.(2020)Kou, Lin, Liu, Li, Zhou, and Zhang]{kou2020disentangle}
Xiaoyu Kou, Yankai Lin, Shaobo Liu, Peng Li, Jie Zhou, and Yan Zhang.
\newblock Disentangle-based continual graph representation learning.
\newblock \emph{arXiv preprint arXiv:2010.02565}, 2020.

\bibitem[Li \& Hoiem(2017)Li and Hoiem]{li2017learning}
Zhizhong Li and Derek Hoiem.
\newblock Learning without forgetting.
\newblock \emph{IEEE transactions on pattern analysis and machine intelligence}, 40\penalty0 (12):\penalty0 2935--2947, 2017.

\bibitem[Liu et~al.(2021)Liu, Yang, and Wang]{liu2021overcoming}
Huihui Liu, Yiding Yang, and Xinchao Wang.
\newblock Overcoming catastrophic forgetting in graph neural networks.
\newblock In \emph{Proceedings of the AAAI Conference on Artificial Intelligence}, volume~35, pp.\  8653--8661, 2021.

\bibitem[Lopez-Paz \& Ranzato(2017)Lopez-Paz and Ranzato]{lopez2017gradient}
David Lopez-Paz and Marc'Aurelio Ranzato.
\newblock Gradient episodic memory for continual learning.
\newblock \emph{Advances in neural information processing systems}, 30, 2017.

\bibitem[Ma et~al.(2021)Ma, Deng, and Mei]{subgroup}
Jiaqi Ma, Junwei Deng, and Qiaozhu Mei.
\newblock Subgroup generalization and fairness of graph neural networks, 2021.

\bibitem[Ma et~al.(2020)Ma, Guo, Ren, Tang, and Yin]{ma2020streaming}
Yao Ma, Ziyi Guo, Zhaocun Ren, Jiliang Tang, and Dawei Yin.
\newblock Streaming graph neural networks.
\newblock In \emph{Proceedings of the 43rd International ACM SIGIR Conference on Research and Development in Information Retrieval}, pp.\  719--728, 2020.

\bibitem[Nguyen et~al.(2018)Nguyen, Lee, Rossi, Ahmed, Koh, and Kim]{nguyen2018continuous}
Giang~Hoang Nguyen, John~Boaz Lee, Ryan~A Rossi, Nesreen~K Ahmed, Eunyee Koh, and Sungchul Kim.
\newblock Continuous-time dynamic network embeddings.
\newblock In \emph{Companion proceedings of the the web conference 2018}, pp.\  969--976, 2018.

\bibitem[Parisi et~al.(2019)Parisi, Kemker, Part, Kanan, and Wermter]{cl_survey2}
German~I Parisi, Ronald Kemker, Jose~L Part, Christopher Kanan, and Stefan Wermter.
\newblock Continual lifelong learning with neural networks: A review.
\newblock \emph{Neural Networks}, 113:\penalty0 54--71, 2019.

\bibitem[Redko et~al.(2020)Redko, Morvant, Habrard, Sebban, and Bennani]{da_survey}
Ievgen Redko, Emilie Morvant, Amaury Habrard, Marc Sebban, and Youn{\`e}s Bennani.
\newblock A survey on domain adaptation theory: learning bounds and theoretical guarantees.
\newblock \emph{arXiv preprint arXiv:2004.11829}, 2020.

\bibitem[Rusu et~al.(2016)Rusu, Rabinowitz, Desjardins, Soyer, Kirkpatrick, Kavukcuoglu, Pascanu, and Hadsell]{rusu2016progressive}
Andrei~A Rusu, Neil~C Rabinowitz, Guillaume Desjardins, Hubert Soyer, James Kirkpatrick, Koray Kavukcuoglu, Razvan Pascanu, and Raia Hadsell.
\newblock Progressive neural networks.
\newblock \emph{arXiv preprint arXiv:1606.04671}, 2016.

\bibitem[Saha et~al.(2021)Saha, Garg, and Roy]{saha2021gradient}
Gobinda Saha, Isha Garg, and Kaushik Roy.
\newblock Gradient projection memory for continual learning.
\newblock \emph{arXiv preprint arXiv:2103.09762}, 2021.

\bibitem[Shin et~al.(2017)Shin, Lee, Kim, and Kim]{shin2017continual}
Hanul Shin, Jung~Kwon Lee, Jaehong Kim, and Jiwon Kim.
\newblock Continual learning with deep generative replay.
\newblock \emph{Advances in neural information processing systems}, 30, 2017.

\bibitem[Su et~al.(2023)Su, Zou, Zhang, and Wu]{su2023towards}
Junwei Su, Difan Zou, Zijun Zhang, and Chuan Wu.
\newblock Towards robust graph incremental learning on evolving graphs.
\newblock 2023.

\bibitem[Su et~al.(2024{\natexlab{a}})Su, Mao, and Wu]{su2024bg}
Junwei Su, Lingjun Mao, and Chuan Wu.
\newblock Bg-hgnn: Toward scalable and efficient heterogeneous graph neural network.
\newblock \emph{arXiv preprint arXiv:2403.08207}, 2024{\natexlab{a}}.

\bibitem[Su et~al.(2024{\natexlab{b}})Su, Zou, and Wu]{su2024pres}
Junwei Su, Difan Zou, and Chuan Wu.
\newblock Pres: Toward scalable memory-based dynamic graph neural networks.
\newblock \emph{arXiv preprint arXiv:2402.04284}, 2024{\natexlab{b}}.

\bibitem[Tan et~al.(2022)Tan, Ding, Guo, and Liu]{tan2022graph}
Zhen Tan, Kaize Ding, Ruocheng Guo, and Huan Liu.
\newblock Graph few-shot class-incremental learning.
\newblock In \emph{Proceedings of the Fifteenth ACM International Conference on Web Search and Data Mining}, pp.\  987--996, 2022.

\bibitem[Wang et~al.(2020{\natexlab{a}})Wang, Qiu, and Scherer]{wang2020bridging}
Chen Wang, Yuheng Qiu, and Sebastian Scherer.
\newblock Bridging graph network to lifelong learning with feature interaction.
\newblock 2020{\natexlab{a}}.

\bibitem[Wang et~al.(2022)Wang, Qiu, Gao, and Scherer]{wang2022lifelong}
Chen Wang, Yuheng Qiu, Dasong Gao, and Sebastian Scherer.
\newblock Lifelong graph learning.
\newblock In \emph{Proceedings of the IEEE/CVF Conference on Computer Vision and Pattern Recognition}, pp.\  13719--13728, 2022.

\bibitem[Wang et~al.(2020{\natexlab{b}})Wang, Song, Wu, and Wang]{wang2020streaming}
Junshan Wang, Guojie Song, Yi~Wu, and Liang Wang.
\newblock Streaming graph neural networks via continual learning.
\newblock In \emph{Proceedings of the 29th ACM International Conference on Information \& Knowledge Management}, pp.\  1515--1524, 2020{\natexlab{b}}.

\bibitem[Wortsman et~al.(2020)Wortsman, Ramanujan, Liu, Kembhavi, Rastegari, Yosinski, and Farhadi]{wortsman2020supermasks}
Mitchell Wortsman, Vivek Ramanujan, Rosanne Liu, Aniruddha Kembhavi, Mohammad Rastegari, Jason Yosinski, and Ali Farhadi.
\newblock Supermasks in superposition.
\newblock \emph{Advances in Neural Information Processing Systems}, 33:\penalty0 15173--15184, 2020.

\bibitem[Wu et~al.(2019)Wu, Chen, Wang, Ye, Liu, Guo, and Fu]{wu2019large}
Yue Wu, Yinpeng Chen, Lijuan Wang, Yuancheng Ye, Zicheng Liu, Yandong Guo, and Yun Fu.
\newblock Large scale incremental learning.
\newblock In \emph{Proceedings of the IEEE/CVF Conference on Computer Vision and Pattern Recognition}, pp.\  374--382, 2019.

\bibitem[Wu et~al.(2020)Wu, Pan, Chen, Long, Zhang, and Philip]{gnn_survey}
Zonghan Wu, Shirui Pan, Fengwen Chen, Guodong Long, Chengqi Zhang, and S~Yu Philip.
\newblock A comprehensive survey on graph neural networks.
\newblock \emph{IEEE transactions on neural networks and learning systems}, 32\penalty0 (1):\penalty0 4--24, 2020.

\bibitem[Xu et~al.(2020)Xu, Zhang, Guo, Guo, Tang, and Coates]{xu2020graphsail}
Yishi Xu, Yingxue Zhang, Wei Guo, Huifeng Guo, Ruiming Tang, and Mark Coates.
\newblock Graphsail: Graph structure aware incremental learning for recommender systems.
\newblock In \emph{Proceedings of the 29th ACM International Conference on Information \& Knowledge Management}, pp.\  2861--2868, 2020.

\bibitem[Yoon et~al.(2017)Yoon, Yang, Lee, and Hwang]{yoon2017lifelong}
Jaehong Yoon, Eunho Yang, Jeongtae Lee, and Sung~Ju Hwang.
\newblock Lifelong learning with dynamically expandable networks.
\newblock \emph{arXiv preprint arXiv:1708.01547}, 2017.

\bibitem[Yoon et~al.(2019)Yoon, Kim, Yang, and Hwang]{yoon2019scalable}
Jaehong Yoon, Saehoon Kim, Eunho Yang, and Sung~Ju Hwang.
\newblock Scalable and order-robust continual learning with additive parameter decomposition.
\newblock \emph{arXiv preprint arXiv:1902.09432}, 2019.

\bibitem[Yu et~al.(2018)Yu, Cheng, Aggarwal, Zhang, Chen, and Wang]{yu2018netwalk}
Wenchao Yu, Wei Cheng, Charu~C Aggarwal, Kai Zhang, Haifeng Chen, and Wei Wang.
\newblock Netwalk: A flexible deep embedding approach for anomaly detection in dynamic networks.
\newblock In \emph{Proceedings of the 24th ACM SIGKDD international conference on knowledge discovery \& data mining}, pp.\  2672--2681, 2018.

\bibitem[Yuan et~al.(2023)Yuan, Guan, Ni, Luo, Man, Wong, and Chang]{yuan2023continual}
Qiao Yuan, Sheng-Uei Guan, Pin Ni, Tianlun Luo, Ka~Lok Man, Prudence Wong, and Victor Chang.
\newblock Continual graph learning: A survey.
\newblock \emph{arXiv preprint arXiv:2301.12230}, 2023.

\bibitem[Zhang et~al.(2021)Zhang, Song, and Tao]{zhang2021hierarchical}
Xikun Zhang, Dongjin Song, and Dacheng Tao.
\newblock Hierarchical prototype networks for continual graph representation learning.
\newblock \emph{arXiv preprint arXiv:2111.15422}, 2021.

\bibitem[Zhang et~al.(2022)Zhang, Song, and Tao]{zhang2022cglb}
Xikun Zhang, Dongjin Song, and Dacheng Tao.
\newblock Cglb: Benchmark tasks for continual graph learning.
\newblock In \emph{Thirty-sixth Conference on Neural Information Processing Systems Datasets and Benchmarks Track}, 2022.

\bibitem[Zhou \& Cao(2021{\natexlab{a}})Zhou and Cao]{cnc_er}
Fan Zhou and Chengtai Cao.
\newblock Overcoming catastrophic forgetting in graph neural networks with experience replay.
\newblock In \emph{Proceedings of the AAAI Conference on Artificial Intelligence}, volume~35, pp.\  4714--4722, 2021{\natexlab{a}}.

\bibitem[Zhou \& Cao(2021{\natexlab{b}})Zhou and Cao]{gcl_er}
Fan Zhou and Chengtai Cao.
\newblock Overcoming catastrophic forgetting in graph neural networks with experience replay.
\newblock In \emph{Proceedings of the AAAI Conference on Artificial Intelligence}, volume~35, pp.\  4714--4722, 2021{\natexlab{b}}.

\bibitem[Zhou \& Cao(2021{\natexlab{c}})Zhou and Cao]{zhou2021overcoming}
Fan Zhou and Chengtai Cao.
\newblock Overcoming catastrophic forgetting in graph neural networks with experience replay.
\newblock In \emph{Proceedings of the AAAI Conference on Artificial Intelligence}, volume~35, pp.\  4714--4722, 2021{\natexlab{c}}.

\bibitem[Zliobaite(2010)]{zliobaite2010learning}
Indre Zliobaite.
\newblock Learning under concept drift: an overview.
\newblock \emph{arXiv preprint arXiv:1010.4784}, 41, 2010.

\end{thebibliography}
\bibliographystyle{collas2024_conference}

\newpage
\appendix
\section{Further Discussion on Learnability and Proof of the Theorem~\ref{thm:divergence_necessity}}\label{appendix:connection_proof}
 In this appendix, we provide a proof for Theorem~\ref{thm:divergence_necessity}. The key idea of the proof is to establish a formal connection between the NGCL-2 problem with the domain adaptation problem (henceforth we refer to it as the DA problem). We begin by giving a brief introduction to the domain adaptation (DA) problem. Without loss of generality, we focus the analysis on binary classification tasks. Obviously, if binary classification is impossible, then an arbitrary k-way classification is also impossible.

\subsection{Domain Adaptation}
Let $\mathcal{X}$ and $\mathcal{Y} = \{0,1\}$ be the input and output space, and $l:\mathcal{X} \mapsto \mathcal{Y}$ be the labelling function. In the DA problem, there are two different distributions over $\mathcal{X}$, the source distribution $\mathcal{S}$ and the target distribution $\mathcal{T}$. The goal of domain adaptation is to learn a classifier $h : \mathcal{X} \mapsto \mathcal{Y}$ that performs well on the target distribution, i.e., $R_{\mathcal{T}}^l(h)$ is small, with the knowledge from the source domain $\mathcal{S}$. Similar to the NGCL-2 problem, the learnability of the DA problem can be defined as follow.


\begin{definition}[DA-learnability]\label{def:da_learnability}
Let $\mathcal{S}$ and $\mathcal{T}$ be distribution over $\mathcal{X}$, $\mathcal{H}$ a hypothesis class, $\epsilon > 0, \delta > 0$. We say that the DA problem is learnable if there exists a learning algorithm that produces a classifier $h$ from $\mathcal{H}$ with probability at least $1-\delta$, and $h$ incurs an error smaller than $\epsilon$, i.e., $Pr[R_{\mathcal{T}}^l(h) \leq \epsilon] \geq 1 - \delta$, when given access to labelled samples $L$ of size $m$ from $\mathcal{S}$ and unlabelled samples $U$ of size $n$ from $\mathcal{T}$.
\end{definition}

\subsection{Connection between DA and NGCL-2}
Recall that The NGCL-2 problem is characterized by three distributions: $\prob(\mathbf{y}, \mathbf{g}_{v}|\vertexSet_1, \graphStruct_{\task_1})$, $\prob(\mathbf{y}, \mathbf{g}_{v}|\vertexSet_1, \graphStruct_{\task_2})$, and $\prob(\mathbf{y}, \mathbf{g}_{v}|\vertexSet_2, \graphStruct_{\task_2})$. To simplify the notation, we re-write $\prob(\mathbf{y}, \mathbf{g}_{v}|\vertexSet_1, \graphStruct_{\task_1})$, $\prob(\mathbf{y}, \mathbf{g}_{v}|\vertexSet_1, \graphStruct_{\task_2})$, $\prob(\mathbf{y}, \mathbf{g}_{v}|\vertexSet_2, \graphStruct_{\task_2})$ as $\prob_1^{(1)}, \prob_1^{(2)}, \prob_2^{(2)}$. Then, the learnability of NGCL-2 problem can be formally defined as,
\begin{definition}[NGCL-2-learnability]\label{def:NGCL2_learnability}
Given a NGCL-2 with distributions $\prob_1^{(1)},\prob_1^{(2)},\prob_2^{(2)}$ and a labelling function $l$, let $\mathcal{H}$ a hypothesis class, $\epsilon > 0, \delta > 0$. We say that the NGCL-2 problem is learnable relative to $\mathcal{H}$ , if there exists a learning algorithm that outputs a classifier $h$ from $\mathcal{H}$ with probability of at least $1-\delta$ and $h$  has error less than $\epsilon$, i.e., $Pr[R_{\prob}^l(h) \leq \epsilon] \geq 1 - \delta$, when given access to labeled samples $L$ from $\prob$ ($\prob$ is the mixed distribution of $\prob_1^{(1)}$ and $\prob_2^{(2)}$) of size $m+n$ and unlabelled samples $U$ of size $k$ from $\prob_1^{(2)}$
\end{definition}

To establish a formal connection between NGCL-2 and DA, we consider a relaxed version of the NGCL-2 problem where the learning algorithm allow to access the same amount of data in an arbitrary order from distribution $\prob_1^{(1)}$ and $\prob_2^{(2)}$ and only needs to perform well on the updated distribution $\prob_1^{(2)}$. We refer to this relaxed version of NGCL-2 problem as NGCL-2-weak and its learnability is defined as follows.

\begin{definition}[NGCL-2-weak-learnability]\label{def:NGCL2_weak_learnability}
Given a NGCL-2 with distributions $\prob_1^{(1)},\prob_1^{(2)},\prob_2^{(2)}$ and a labelling function $l$, let $\mathcal{H}$ a hypothesis class, $\epsilon > 0, \delta > 0$. We say that the NGCL-2 problem is learnable relative to $\mathcal{H}$ , if there exists a learning algorithm that outputs a classifier $h$ from $\mathcal{H}$ with probability of at least $1-\delta$ and $h$  has error less than $\epsilon$, i.e., $Pr[R_{\prob_1^{(2)}}^l(h) \leq \epsilon] \geq 1 - \delta$, when given access to labeled samples $L$ from $P$ ($P$ is the mixed distribution of $\prob_1^{(1)}$ and $\prob_2^{(2)}$) of size $m+n$ and unlabelled samples $U$ of size $k$ from $\prob_1^{(2)}$
\end{definition}

It is obvious that the NGCL-2-weak-learnability definition subsume the NGCL-2-learnability. In other words, we can obtain the following corollary immediately from the definition.

\begin{corollary}\label{cor:weak_connection}
 If the NGCL-2-weak problem is not learnable, then the NGCL-2 problem is not learnable.
\end{corollary}

Next, we establish a formal connection between the DA problem and the NGCL-2-weak problem.

\begin{theorem}[Connection of DA and NGCL-2-weak] \label{thm:connection}
Let $\mathcal{X}$ be a domain and $\mathcal{H}$ the hypothesis class on $\mathcal{X} \times \{0,1\}$. Suppose an instance of the DA problem with distributions $\mathcal{S}$ and $\mathcal{T}$ on $\mathcal{X}$, and there is access to $m'$ labelled samples from $\mathcal{S}$ and $k'$ unlabelled samples from $\mathcal{T}$. Similarly, suppose an instance of an NGCL-2-weak problem with distributions $\prob_1^{(1)},\prob_1^{(2)},\prob_2^{(2)}$ over $\mathcal{X}$,and there is access to $m+n$ labelled samples from $\prob_1^{(1)},\prob_1^{(2)}$ and $k$ unlabelled samples from $\prob_1^{(2)}$.  For the same hypothesis sapce $\mathcal{H}$, assuming that $m'=m+n, k'=k$, $\lambda_{\mathcal{H}}(\{\mathcal{S},\mathcal{T}\}) = \lambda_{\mathcal{H}}(\{P, \prob_1^{(2)}\})$ and $d_{\mathcal{H}\Delta \mathcal{H}}(\mathcal{S},\mathcal{T}) = d_{\mathcal{H}\Delta \mathcal{H}}(P, \prob_1^{(2)})$, then we have that the DA is learnable  if and only if the NGCL-2-weak problem is learnable.
\end{theorem}

\begin{proof}
We first show that under the premise of the theorem if the NGCL-2-weak learnable, then the DA problem is learnable.

Suppose we are given an instance of DA problem of distributions $\mathcal{S}$, $\mathcal{T}$ and a $\mathcal{H}$. Let $\mathcal{S}$ be a mixed distribution of $\mathcal{S}_1$ and $\mathcal{S}_2$. Then, we can map the DA to a NGCL-2-weak problem as follows.
$$\mathcal{S}_1 = \prob_1^{(1)},\mathcal{S}_2 = \prob_1^{(2)},\mathcal{T} = \prob_1^{(2)}$$

Under the mapping above, we have $d_{\mathcal{H}\Delta \mathcal{H}}(\mathcal{S},\mathcal{T}) = d_{\mathcal{H}\Delta \mathcal{H}}(P,\prob_1^{(2)}) $ and $\lambda_{\mathcal{H}}(\{\mathcal{S},\mathcal{T}\}) = \lambda_{\mathcal{H}}(\{P, \prob_1^{(2)}\})$, where $P$ is the mixed distribution of distributions $\prob_1^{(1)}$ and $\prob_1^{(2)}$. 

By the assumption, the NGCL-2-weak problem is learnable. By Def.~\ref{def:NGCL2_weak_learnability}, this means that there exist a learning algorithm that takes in $m+n$ labelled data from $P$ and $k$ unlabelled data from $\prob_1^{(2)}$, and with probability at least $1-\delta$, outputs a classifier $h$ wiht at most $\epsilon$ error with respect to $\prob_1^{(2)}$. By the mapping above and the premise of the theorem that $m+n = m', k'=k$, the same learning algorithm would satisfy Def.~\ref{def:da_learnability}. This shows that the learnability of NGCL-2-weak implies the learnability of DA problem.

Now we show the opposite that  under the premise of the theorem if the DA problem is learnable, then the NGCL-2-weak learnable. 

Let's reverse the construction direction. Suppose we are given an instance of NGCL-2-weak problem of distributions $\prob_1^{(1)}$, $\prob_1^{(2)}$, $\prob_1^{(2)}$ and a $\mathcal{H}$. Let $\mathcal{P}$ be a mixed distribution of $\prob_1^{(1)}$ and $\prob_1^{(2)}$. Then, we can map the NGCL-2-weak problem to DA problem as follows.
$$P = \mathcal{S},\prob_1^{(2)} = \mathcal{T}$$

Similarly, under the mapping above, we have $d_{\mathcal{H}\Delta \mathcal{H}}(\mathcal{S},\mathcal{T}) = d_{\mathcal{H}\Delta \mathcal{H}}(P,\prob_1^{(2)}) $ and $\lambda_{\mathcal{H}}(\{\mathcal{S},\mathcal{T}\}) = \lambda_{\mathcal{H}}(\{P, \prob_1^{(2)}\})$.

By the assumption, the DA problem is learnable. By Def.~\ref{def:da_learnability}, this means that there exist a learning algorithm that takes in $m'$ labelled data from $\mathcal{S}$ and $k'$ unlabelled data from $\mathcal{T}$, and with probability at least $1-\delta$, outputs a classifier $h$ wiht at most $\epsilon$ error with respect to $\mathcal{T}$. By the mapping above and the premise of the theorem that $m+n = m', k'=k$, the same learning algorithm would satisfy Def.~\ref{def:NGCL2_weak_learnability}. This shows that the learnability of DA problem implies the learnability of the NGCL-2-weak problem.

This completes the proof of Theorem~\ref{thm:connection}.
\end{proof}

Now that we have shown that the learnability of DA problem is equivalent to the learnability of the NGCL-2-weak problem. Combining Theorem~\ref{thm:connection} and Corollary~\ref{cor:weak_connection}, we immediately obtain the following corollary on the connection between the NGCL-2 problem with Def.~\ref{def:NGCL2_learnability} and the domain adaptation problem.

\begin{corollary}[Connection of DA and NGCL-2] \label{cor:connection}
Let $\mathcal{X}$ be a domain and $\mathcal{H}$ the hypothesis class on $\mathcal{X} \times \{0,1\}$. Suppose an instance of the DA problem with distributions $\mathcal{S}$ and $\mathcal{T}$ on $\mathcal{X}$, and there is access to $m'$ labelled samples from $\mathcal{S}$ and $k'$ unlabelled samples from $\mathcal{T}$. Similarly, suppose an instance of a NGCL-2 problem with distributions $\prob_1^{(1)},\prob_1^{(2)},\prob_1^{(2)}$ over $\mathcal{X}$,and there is sequential access to $m$ labelled samples from $\prob_1^{(1)}$, $n$ labelled samples from $\prob_1^{(2)}$ and $k$ unlabelled samples from $\prob_1^{(2)}$.  For the same hypothesis sapce $\mathcal{H}$, assuming that $m'=m+n, k'=k$, $\lambda_{\mathcal{H}}(\{\mathcal{S},\mathcal{T}\}) = \lambda_{\mathcal{H}}(\{P, \prob_1^{(2)}\})$ and $d_{\mathcal{H}\Delta \mathcal{H}}(\mathcal{S},\mathcal{T}) = d_{\mathcal{H}\Delta \mathcal{H}}(P, \prob_1^{(2)})$, then we have that the if the DA is not learnable, then the NGCL-2 problem is not learnable.
\end{corollary}

\subsection{Proof for Theorem~\ref{thm:divergence_necessity}}
Corollary~\ref{cor:connection} above establishes a formal connection between the DA problem and the NGCL-2 problem. Under the same proposed mapping, the infeasibility of DA implies the infeasibility of NGCL-2. DA is a well-studied problem with mature understanding on its learnability. 

\cite{impossibility_expressive_divergence} has shown that small divergent distance between source domain and target domain is a ``necessary'' condition for the success of domain adaptation. We restate the theorem below.

\begin{theorem}[~\citep{impossibility_expressive_divergence}]\label{thm:da_divergence}
Let $\mathcal{X}$ be some domain set, and $\mathcal{H}$ a class of functions over
$\mathcal{X} \times \{0,1\}$. For every $c > 0$ there exists probability distributions $\mathcal{S}, \mathcal{T}$ over $\mathcal{X}$ such that for every domain adaptation learner, every integers $m, n >$ 0, there exists a labeling function $l : \mathcal{X} \mapsto {0, 1}$ such that $\lambda_{\mathcal{H}}(P,Q) \leq c$ and the DA problem is not learnable.
\end{theorem}

The proof of Theorem~\ref{thm:da_divergence} can be found in ~\citep{impossibility_expressive_divergence}. Then the proof of Theorem~\ref{thm:divergence_necessity} follows immediately from Theorem~\ref{thm:da_divergence} and Corollary~\ref{cor:connection}.

\section{Expressiveness and Sample Complexity}\label{appendix:sample_expressiveness}
In this appendix, building upon the connection established in the previous appendix, we provide a further discussion of the sample complexity and expressiveness.

\begin{theorem}[Necessity of the Expressiveness of Hypothesis Space]\label{thm:express_necessity}
Let $\mathcal{X}$ be some domain set, and $\mathcal{H}$ be a class of functions over $\mathcal{X} \times \{0,1\}$ whose VC dimension is much smaller than $|\mathcal{X}|$. Then for every $c > 0$, there exists a NGCL-2 problem such that $d_{\mathcal{H}\Delta \mathcal{H}}(\prob,\prob_1^{(2)}) \leq c$, where $\prob$ is the mixed distribution of $\prob_1^{(1)}$ and $\prob_2^{(2)}$, and the NGCL-2 problem is not learnable with any amount of labelled data from $\prob_1^{(1)}, \prob_2^{(2)}$ and unlablled data from $\prob_1^{(2)}$.
\end{theorem}

Theorem~\ref{thm:express_necessity} formalizes the intuition that the hypothesis class (GNN model) needs to be expressive enough to capture different distributions in NGCL, for it to be learnable. 

\begin{theorem}[Sample Complexity]\label{thm:sample_complexity}
Let $\mathcal{X}$ be some compact domain set in $\mathbf{R}^d$, and $\mathcal{H}$ be a class of functions over $\mathcal{X} \times \{0,1\}$. Suppose the labeling function $l: \mathcal{X} \mapsto \{0,1\}$ is $\alpha$-Lipschitz continuous. Then,  if $m + n + k < \sqrt{(1- 2(\epsilon+\delta)) (\alpha + 1)^d}$, for every $c > 0$, there exists a NGCL-2 problem such that $\lambda_{\mathcal{H}}(\{\prob_1^{(1)}, \prob_2^{(2)}, \prob_1^{(2)}\}) \leq c,$ 
$d_{\mathcal{H}\Delta \mathcal{H}}(P,\prob_1^{(2)}) \leq c$, where $\prob$ is the mixed distribution of $\prob_1^{(1)}$ and $\prob_2^{(2)}$, and the NGCL-2 problem is not learnable.
\end{theorem}

Theorem~\ref{thm:sample_complexity} suggests that even if the two ``necessary'' conditions in Theorem~\ref{thm:divergence_necessity} and Theorem~\ref{thm:express_necessity} are satisfied, there is still a sampling requirement depending on the dimension of the input domain and the smoothness of the labelling function. This shows the importance of having access to previous data for NGCL.

\subsection{Proof for Theorem~\ref{thm:express_necessity} and Theorem~\ref{thm:sample_complexity}}
The idea of the proof is similar to the one presented earlier. ~\citep{impossibility_expressive_divergence} has shown that expressive power of the hypothesis space is a ``necessary'' condition for the success of domain adaptation. We restate the theorem below.

\begin{theorem}\label{thm:da_expressive}
Let $\mathcal{X}$ be some domain set, and $\mathcal{H}$ a class of functions over
$\mathcal{X} \times \{0,1\}$ whose VC dimension is much smaller than $|\mathcal{X}|$. Then for every $c > 0$ there exists probability distributions $Q, P$ over $\mathcal{X}$ such that for every domain adaptation learner, every integers $m, n >$ 0, there exists a labeling function $l : \mathcal{X} \mapsto {0, 1}$ such that $d_{\mathcal{H} \Delta \mathcal{H}}(P,Q) \leq c$ and the DA problem is not learnable.
\end{theorem}

The proof of Theorem~\ref{thm:da_expressive} can be found in ~\citep{impossibility_expressive_divergence}. Then the proof of Theorem~\ref{thm:express_necessity} follows immediately from Theorem~\ref{thm:da_expressive} and Corollary~\ref{cor:connection}.

~\citep{sample_complexity} has shown that even if DA problem has small divergent distance and the hypothesis space is expressive enough, there is still a sampling requirement for the DA problem to be learnable. We restate the theorem below.

\begin{theorem}\label{thm:da_sample_complexity}
Let $\mathcal{X} \subset \mathbf{R}^d$ and $\mathcal{H}$ a class of functions over
$\mathcal{X} \times \{0,1\}$. Suppose the labelling function $l$ is $\alpha$-Lipschitz continuous. For every $c > 0$ there exist $\mathcal{S}, \mathcal{T}$ over $\mathcal{X}$ suc htat $ \lambda_{\mathcal{H}} \leq c$ and $d_{\mathcal{H}\Delta \mathcal{H}}(\mathcal{S},\mathcal{T}) \leq c$ and the DA problem is not learnable if $m + n \geq \sqrt{(\alpha + 1)^d(1-2(\epsilon+\delta))}$, where $m,n,\epsilon,\delta$ are given in Def.~\ref{def:da_learnability}.
\end{theorem}

The proof of Theorem~\ref{thm:da_sample_complexity} can be found in ~\citep{sample_complexity}. Then the proof of Theorem~\ref{thm:sample_complexity} follows immediately from Theorem~\ref{thm:da_sample_complexity} and Corollary~\ref{cor:connection}.

\begin{remark}
In addition to the learnability studies, there exist a large body of studies investigating different techniques for domain adaptation. The formal connection~\ref{cor:connection} allow to rigorously connection the NGCL-2 problem and DA problem and to carefully transfer or explore the understanding and techniques from DA to NGCL-2. 
\end{remark}



\section{Proof of Proposition~\ref{prop:performance_guarantee}}\label{appendix:structural_relation}
In this appendix, we provide proof for Proposition~\ref{prop:performance_guarantee}. We begin with restating the definition of distortion. 

\begin{definition}[distortion rate]\label{def:distortion}
    Given two metric spaces $(\mathcal{Q},d)$ and $(\mathcal{Q'},d')$ and  
 	a mapping $\gamma: \mathcal{Q} \mapsto \mathcal{Q'}$, $\gamma$ is said to have a distortion $\alpha \geq 1$, if there exists a constant $r > 0$ such that $\forall u,v \in \mathcal{E}$, $ r d(u,v) \leq d'(\gamma(u),\gamma(v))\leq \alpha r d(u,v)$.
\end{definition}

\begin{proof}
Suppose $P_i \subset \vertexSet_i$ is the experience replay set selected by solving Eq.~\ref{eq:experience_buff}. Let $s$ be the structure of interest with distance measure $d_s$. Let $\gnnModel$ be a given GNN model with distortion $\alpha$ and scaling factor $r$ and its prediction function $g$.

Let $\sigma:\vertexSet_i \mapsto P_i$ denote a mapping that map a vertex $v$ from $\vertexSet_i$ to the closest vertex in $P_i$, i.e.,
\begin{equation}
     \sigma(v) = \argmin_{u \in P_i} d_{s}(u,v).
\end{equation}

Let's consider the loss of vertex $v$, $\loss(g \circ \gnnModel(v), y_v)$. By assumption, $\loss$ is smooth and let's denote $B_{\loss}^{up1}$ to be the upper bound for the first derivative with respect to $g \circ \gnnModel(v)$ and $B_{\loss}^{up2}$ to be the upper bound for the first derivative with respect to $y_v$. Then, consider the Taylor expansion of $\loss(g \circ \gnnModel(v), y_v)$ with respect to $\sigma(v)$, which is given as follows
\begin{equation}
\begin{split}
    \loss(g \circ \gnnModel(v), y_v) & \leq \loss(g \circ \gnnModel(\sigma(v)), y_{\sigma(v)}) + \\ & B_{\loss}^{up1}|| g \circ \gnnModel(v) - g \circ \gnnModel(\sigma(v))|| \\
    & + B_{\loss}^{up2}||y_v - y_{\sigma(v)}||
\end{split}
\end{equation}

Next, let's examine the inequality above term by term and start with $|| g \circ \gnnModel(v) - g \circ \gnnModel(\sigma(v))||$. Let's denote $h_v$ and $h_{\sigma(v)}$ the embedding for vertex $v$ and $\sigma(v)$. Then, we have that,
\begin{equation}
    || g \circ \gnnModel(v) - g \circ \gnnModel(\sigma(v))|| = || g(h_v) - g( h_{\sigma(v)})||
\end{equation}
By definition of distortion as given in Def.~\ref{def:distortion}, we have that,
\begin{equation}
    ||h_v - h_{\sigma(v)}|| \leq r \alpha d_s(v, \sigma(v))
\end{equation}
By assumption, the prediction function $g$ is smooth. Let $B_{g}$ denote the upper bound of first derivative of the prediction function $g$. Then, we have 
\begin{equation}
    || g(h_v) - g(h_{\sigma(v)})|| \leq B_g^{up} ||h_v - h_{\sigma(v)}||
\end{equation}
Substitute all these back, we have
\begin{equation}
     || g \circ \gnnModel(v) - g \circ \gnnModel(\sigma(v))|| \leq B_g^{up} r \alpha d_s(v, \sigma(v))
\end{equation}
Next, let's consider $||y_v - y_{\sigma(v)}||$. Similarly, by the data smoothness assumption and distortion, we have that 
\begin{equation}
    ||y_v - y_{\sigma(v)}|| \leq B_l ||h_v - h_{\sigma(v)}|| \leq B_l^{up} r \alpha d_s(v, \sigma(v))
\end{equation}

Substitute these back to the inequality we start with, we have that
\begin{equation}
    \begin{split}
           & \loss(g \circ \gnnModel(v), y_v)  \leq \loss(g \circ \gnnModel(\sigma(v)), y_{\sigma(v)}) + \\
           & \quad B_{\loss}^{up1}B_g^{up} r \alpha d_s(v, \sigma(v)) + B_{\loss}^{up2}B_l^{up} r \alpha d_s(v, \sigma(v)) \\
           & \leq \loss(g \circ \gnnModel(\sigma(v)), y_{\sigma(v)}) + \\                        
           &(B_l^{up}B_{\loss}^{up1}+B_g^{up}B_{\loss}^{up2}) r \alpha \max_{v' \in \vertexSet_i} d_s(v', \sigma(v'))\\
           & = \loss(g \circ \gnnModel(\sigma(v)), y_{\sigma(v)}) + \\ 
           &(B_l^{up}B_{\loss}^{up1}+B_g^{l}B_{\loss}^{up2}) r \alpha D_{s}(\vertexSet_i, P_i).
    \end{split}
\end{equation}

Since the inequality above holds for every vertex $v$, then we have 

\begin{equation}
    \begin{split}
         R^{\loss}(g \circ \gnnModel,\vertexSet_i) & \leq R^{\loss}(g \circ \gnnModel,P_i) + \\ 
         &(B_l^{up}B_{\loss}^{up1}+B_g^{up}B_{\loss}^{up2})  r \alpha D_{s}(\vertexSet_i, P_i)\\
         & = R^{\loss}(g \circ \gnnModel,P_i) + \mathcal{O} (\alpha  D_s(\vertexSet_i,P_i))
    \end{split}
\end{equation}

\end{proof}
\section{Additional Experimental Details}\label{appendix:exp}
In this appendix, we provide additional experimental results and include a detailed set-up of the experiments for reproducibility.

\subsection{Hardware and System}
All the experiments of this paper are conducted on the following machine

CPU: two Intel Xeon Gold 6230 2.1G, 20C/40T, 10.4GT/s, 27.5M Cache, Turbo, HT (125W) DDR4-2933

GPU: four NVIDIA Tesla V100 SXM2 32G GPU Accelerator for NV Link

Memory: 256GB (8 x 32GB) RDIMM, 3200MT/s, Dual Rank

OS: Ubuntu 18.04LTS

\subsection{Dataset and Processing}
\subsubsection{Dataset Description}
{\bf OGB-Arxiv.} The OGB-Arxiv dataset~\cite{ogb} is a benchmark dataset for node classification. It is constructed from the arXiv e-print repository, a popular platform for researchers to share their preprints. The graph structure is constructed by connecting papers that cite each other. The node features include the text of the paper's abstract, title, and its authors' names. Each node is assigned one of 40 classes, which corresponds to the paper's main subject area. 

{\bf Cora-Full.} The Cora-Full~\cite{cora_full} is a benchmark dataset for node classification. Similarly to OGB-Arxiv, it is a citation network consisting of 70 classes. 

{\bf Reddit.} The Reddit dataset~\cite{reddit} is a benchmark dataset for node classification that consists of posts and comments from the website Reddit.com. Each node represents a post or comment and each edge represents a reply relationship between posts or comments.

\begin{table}[h!]
  \centering
     \vspace{-3mm}
  \caption{Continual learning settings for each dataset.
  }  
  {\large
   \setlength\tabcolsep{4pt}
    \begin{tabular}{c|ccc}
    \hline
        Datasets  &        OGB-Arxiv  & Reddit & CoraFull          \\
            
    \hline
           \# vertices & 169,343 & 227,853 & 19,793\\
           \# edges   & 1,166,243 & 114,615,892 & 130,622 \\
           \# class & 40 & 40 & 70\\
    \hline
         \# tasks & 20 & 20 & 35\\
         \# vertices / \# task & 8,467 & 11,393 & 660\\
         \# edges / \# task & 58,312 & 5,730,794 & 4,354\\
    \hline
    \end{tabular}%
   }
   \vspace{-3mm}
  \label{tab:data_description2}%
\end{table}%

\subsubsection{License}
The datasets used in this paper are curated from existing public data sources and follow their licenses. OGB-Arxiv is licensed under Open Data Commons Attribution License (ODC-BY).  Cora-Full dataset and the Reddit dataset  are two datasets built from publicly available sources (public papers and Reddit posts) without a license attached by the authors.

\subsubsection{Data Processing}
For the datasets, we remove the 41-th class of Reddit-CL, following closely in ~\cite{zhang2022cglb}. This aims to ensure an even number of classes for each dataset to be divided into a sequence of 2-class tasks. For all the datasets, the train-validation-test splitting ratios
are 60\%, 20\%, and 20\%. The train-validation-test
splitting is obtained by random sampling, therefore the performance may be slightly different with splittings from different rounds of random sampling.

\subsection{Continual Training Procedure}
Suppose $\task_1, \task_2,...\task_m$ is the set of classification tasks created from the dataset. Then, the standard continual training procedure is given in Alg.~\ref{alg:it} and a graphical illustration is given in Fig.~\ref{fig:it}.

\subsection{Contextual Stochastic Block Model}\label{appendix:sbm}
We use the contextual stochastic block model~\citep{cSBM} to create a graph for each learning stage. For each graph, we create 300 vertices and assign them a community label $\{0,1\}$. We refer to a specific ratio of vertices from different communities as a configuration and the formula we use to vary the community ratio is as follows, configuration $i = (150+(i-1)*25,150-(i-1)*25,150-(i-1)*25,150+(i-1)*25)$. As $i$ increases, the difference in the ratio of the two learning stages increases. Intuitively, this should allow us to control the distribution difference between different learning stages.

Then we follow the standard stochastic block model. We use $p_{intra} = 0.15$ to randomly create edges among vertices of the same community and $p_{inter} = 0.1$ to randomly create edges among vertices of different communities. Then, we use $p_{stage} = [0.025,0.05,0.08,0.1,0.15]$ to create edges among vertices of different learning stages.

Then, the node features $x_v$ of each vertex is created based on the following formula,
$$x_v = \sqrt{\frac{\mu}{n}} y_v u + \frac{Z_v}{p},$$
where $y_v \in \{0,1\}$ is the community label of vertex $v$,  $u, Z_v \sim \mathcal{N}(0, I)$ are random Gaussian of dimension $p$, $n$ is the number of vertices in the graph and $\mu$ is a hyper-parameter of the model to control how separable of the features of the vertices from different communities. For our experiments, we use $\mu=5, n=300, p=500$ for the graph in each stage. 

\subsection{Graph Neural Network Details}
For our experiments, we use two popular GNN models, GCN~\citep{gcn} and GraphSage~\citep{sage}. For implementation, we use the default implementation provided by DGL for realizing these models. For all our experiments, we use two GNN layers followed by an additional readout function of an MLP layer. We use Sigmoid as the activation function for the last hidden layer and Relu for the rest of the layers. We only vary the hidden dimension of the GNN layer for each experiment.

\subsection{Difference between Inductive and Transductive}
In this subsection, we provide the additional results on the remaining dataset for the setting with evolving graphs and constant graphs and additional illustrations on the difference between transductive and inductive NGCL (with structural shift). The additional experimental results are reported in Fig.~\ref{fig:addition1} and ~\ref{fig:addition2}. The graphical illustration of the difference between transductive and inductive NGCL is given in Fig.~\ref{fig:it_diff}

\subsection{Permutation of Task Order}
In the empirical study, the task order is constructed once for each dataset. The reasons behind this are as follows. The focused continual learning strategy in this paper is experience replay, where the model is presented with (a small subset of) previous tasks regardless of their appearance order. Furthermore, the performance measures we used in the paper are FAP (Final Average Performance) and FAF (Final Average Forgetting).  These metrics measure the performance of the model after being presented with all the tasks. As such, the order of the task permutation might affect the intermediate learning dynamic but should not affect the performance of the experience reply with the two chosen metrics. This is further validated by Fig.~\ref{fig:permutation} shows the permutation of task order has minimal effect on the performance of experience replay on the chosen metrics.

\begin{figure}[!h]
    \centering
    \subfigure[Reddit, Forggeting, Evolving Graphs]{
    \includegraphics[width=.4\textwidth]{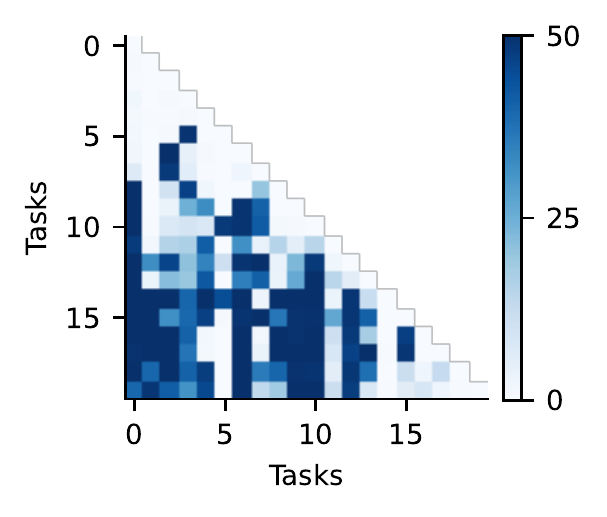}
    }
    \subfigure[Reddit, Forggeting, Constant Graphs]{
    \includegraphics[width=.4\textwidth]{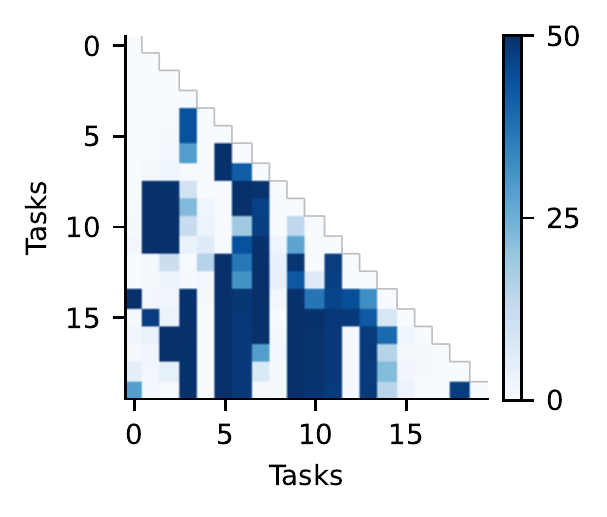}
    }
    \subfigure[CoraFull, Forggeting, Evolving Graphs]{
    \includegraphics[width=.4\textwidth]{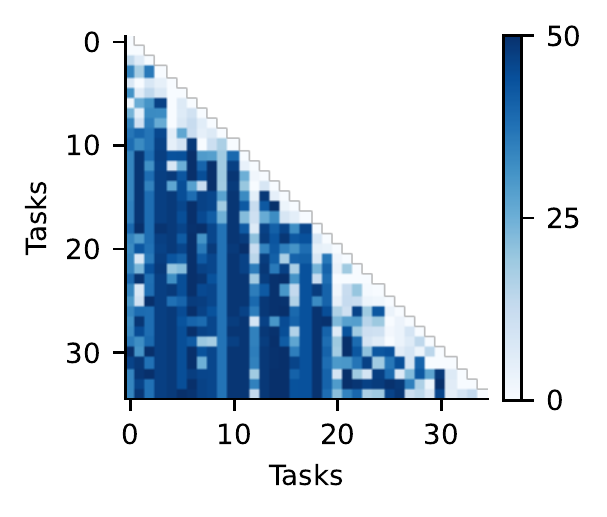}
    }
    \subfigure[CoraFull, Forggeting, Constant Graphs]{
    \includegraphics[width=.4\textwidth]{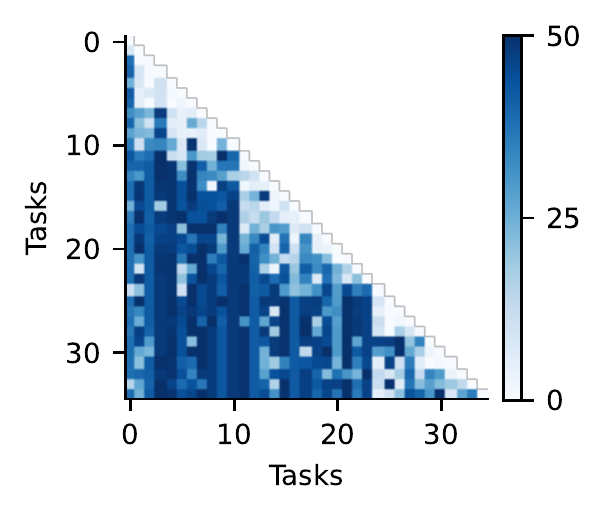}
    }
    \caption{Performance Matrix of Bare Model on Transductive and Inductive Setting on CoraFull and Reddit Datasets.}
    \label{fig:addition1}
\end{figure}

\begin{figure}[!h]
    \centering
    \subfigure[SEA-ER, Coral-Full]{
    \includegraphics[width=.45\textwidth]{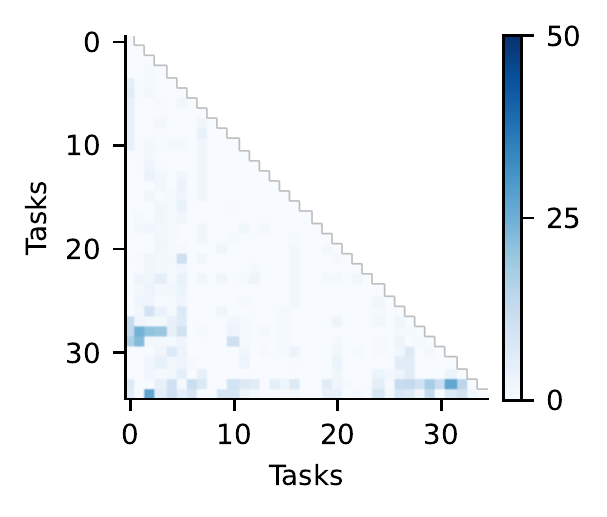}
    }
    \subfigure[SEA-ER, OGB-Arixv]{
    \includegraphics[width=.45\textwidth]{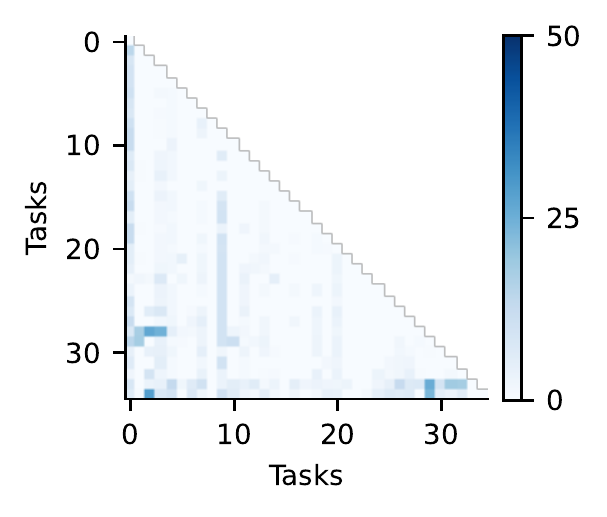}
    }
    \subfigure[SEA-ER, Reddit]{
    \includegraphics[width=.45\textwidth]{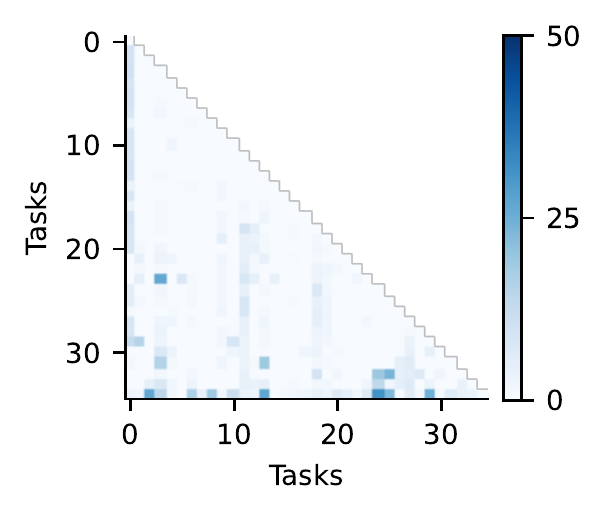}
    }
    \caption{Forgetting Performance Matrix of SEA-ER on CoraFull and Reddit Datasets.}
    \label{fig:addition2}
\end{figure}

\begin{figure}
    \centering
    \includegraphics[width=0.45\linewidth]{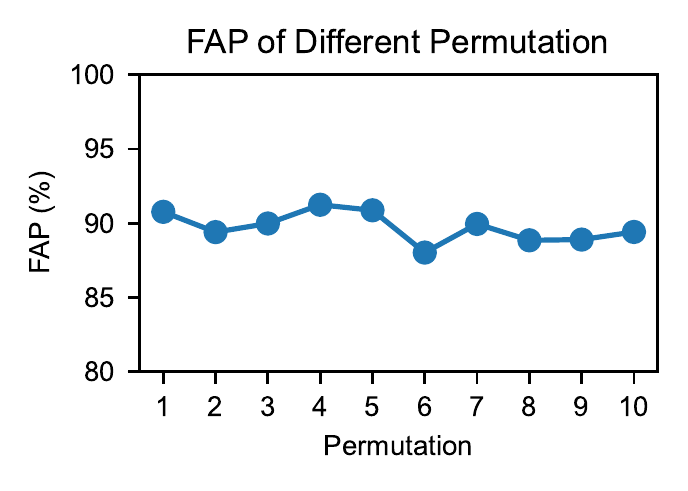}
    \caption{FAP performance of our method with respect to different permutations. X-axis indicate different permutation 1-10. The figure shows that different permutations have a small effect (~1\%) on the performance.}
    \label{fig:permutation}
\end{figure}

\begin{algorithm}[!h]
  \caption{Continual Training}
  \label{alg:it}
\begin{algorithmic}
  \STATE {\bfseries Input:} $\task_1, \task_2,...\task_m$ //classification task created from the dataset. 
  
  For $\tau$ in $\task_1, \task_2,...\task_m$:\\
  \begin{enumerate}
      \item update the graph structure with data from $\tau_i$
      \item get $V^{\mathrm{train}}_i$ for $\tau_i$
      \item update the existing GNN model and train a new task prediction head on $V^{\mathrm{train}}_i$
      \item if evaluation, evaluate the model on the current and previous task 
  \end{enumerate}
\end{algorithmic}
\end{algorithm}

\begin{figure}[!h]
\centering
\includegraphics[width=0.8\textwidth]{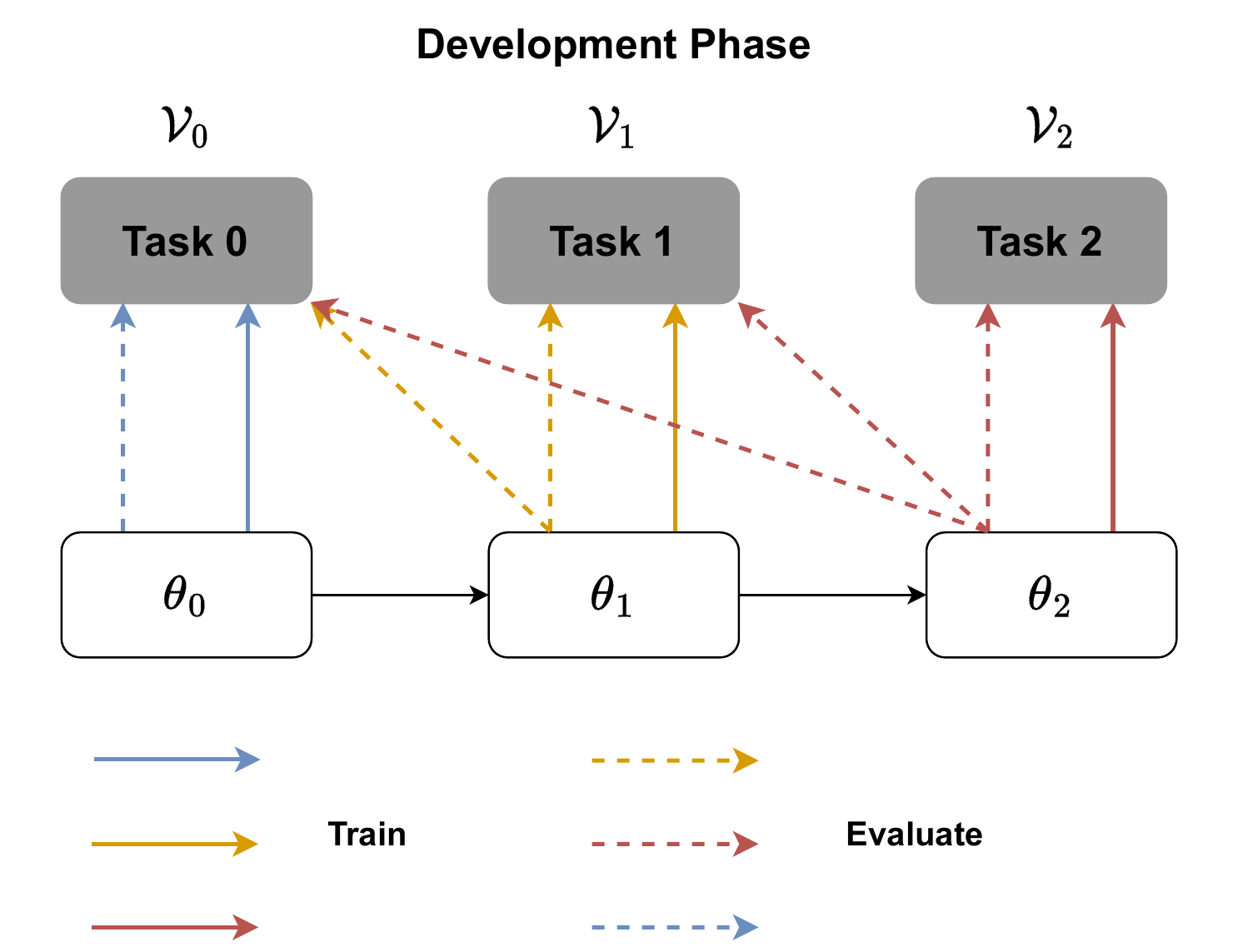}
\caption{Continual Training in Development Phase. The figure above illustrates what task data the model use for training and evaluation.
}
\label{fig:it}
\end{figure}

\begin{figure*}[!t]
\centering
\includegraphics[width=1.0\textwidth]{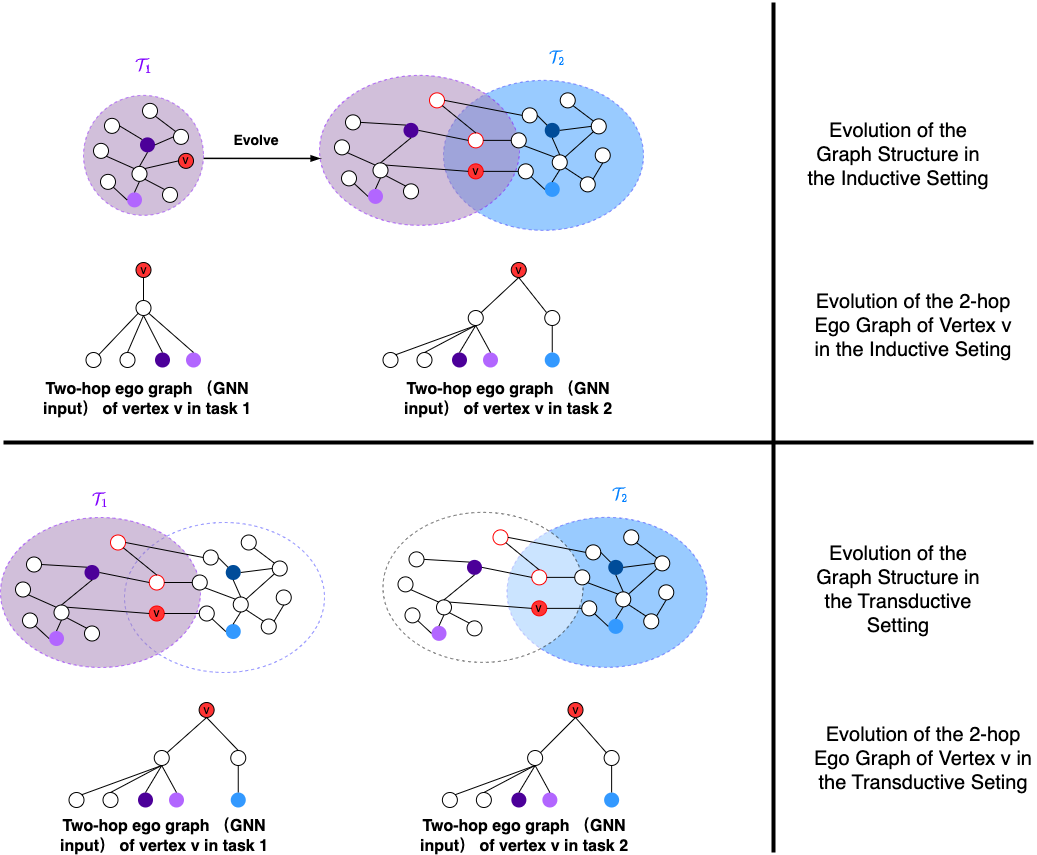}
\caption{Difference between Inductive (evolving graphs) GCL and Transductive (constant graphs) GCL. As illustrated in the top half of the figure, in the evolving graph setting, the appearing tasks expand/change the existing graph structure, resulting in changes in the neighborhood (input to the GNN) of some previous vertices. As discussed in the main body of the paper, this structural changes/dependency would cause distribution shift in the previously learnt information of the model. On the other hand, in the constant graph setting, the complete graph structure is available from the very first beginning. The neighborhoods of all the vertices are static/stay unchanged throughout different learning stages. Therefore, structural dependency is a unique challenge in the continual learning in evolving graphs.
}
\label{fig:it_diff}
\end{figure*}

\section{Complete Framework and Pesudo-code}\label{appendix:framework}

\subsection{Overall framework}
The continual learning procedure of our framework is summarized in Alg.~\ref{alg:train} and Fig.~\ref{fig:sail} provide a graphical illustration of the training of the framework on a two-stage continual learning.

\begin{algorithm}[!h]
  \caption{SEA-ER-GNN}
  \label{alg:train}
\begin{algorithmic}
  \STATE {\bfseries Input:} $\task_i$ //new node classification task\\
  \REQUIRE $\mathcal{P}_{i-1} = \{P_1,P_2,...,P_{i-1}\}$ //replay buffer  
  \REQUIRE $\gnnModel$ //current GNN model
  \REQUIRE $\mathcal{G}_{\task_{i-1}}$ //graph induced by vertices of $\task_{i-1}$ \\
  \REQUIRE $\beta_l, \beta_u$
   
  Create the training set for this learning stage
  $$\hat{\vertexSet}_i = \vertexSet^{\mathrm{train}} \cup \mathcal{P}_{i-1};$$

 Compute $\beta$ with Eq.~\ref{eq:kmm_weight}
  
  Run the training procedure (forward computation and backward propagation) with loss function 
\begin{equation*}
    \loss = \frac{1}{|\vertexSet^{\mathrm{train}}_i|} \sum_{v \in \vertexSet^{\mathrm{train}}_i} \loss(v) + \sum_{P_j \in \experienceBuffer_i}\frac{1}{|P_j|} \sum_{v \in P_j}\beta_v \loss(v),
\end{equation*}
  
  Update the replay buffer with replay set $P_i$ from $\tau_i$
  $$\mathcal{P}_i = \mathcal{P}_{i-1} \cup P_i ; $$
\end{algorithmic}
\end{algorithm}

\begin{figure}[!t]
\centering
\includegraphics[width=0.8\textwidth]{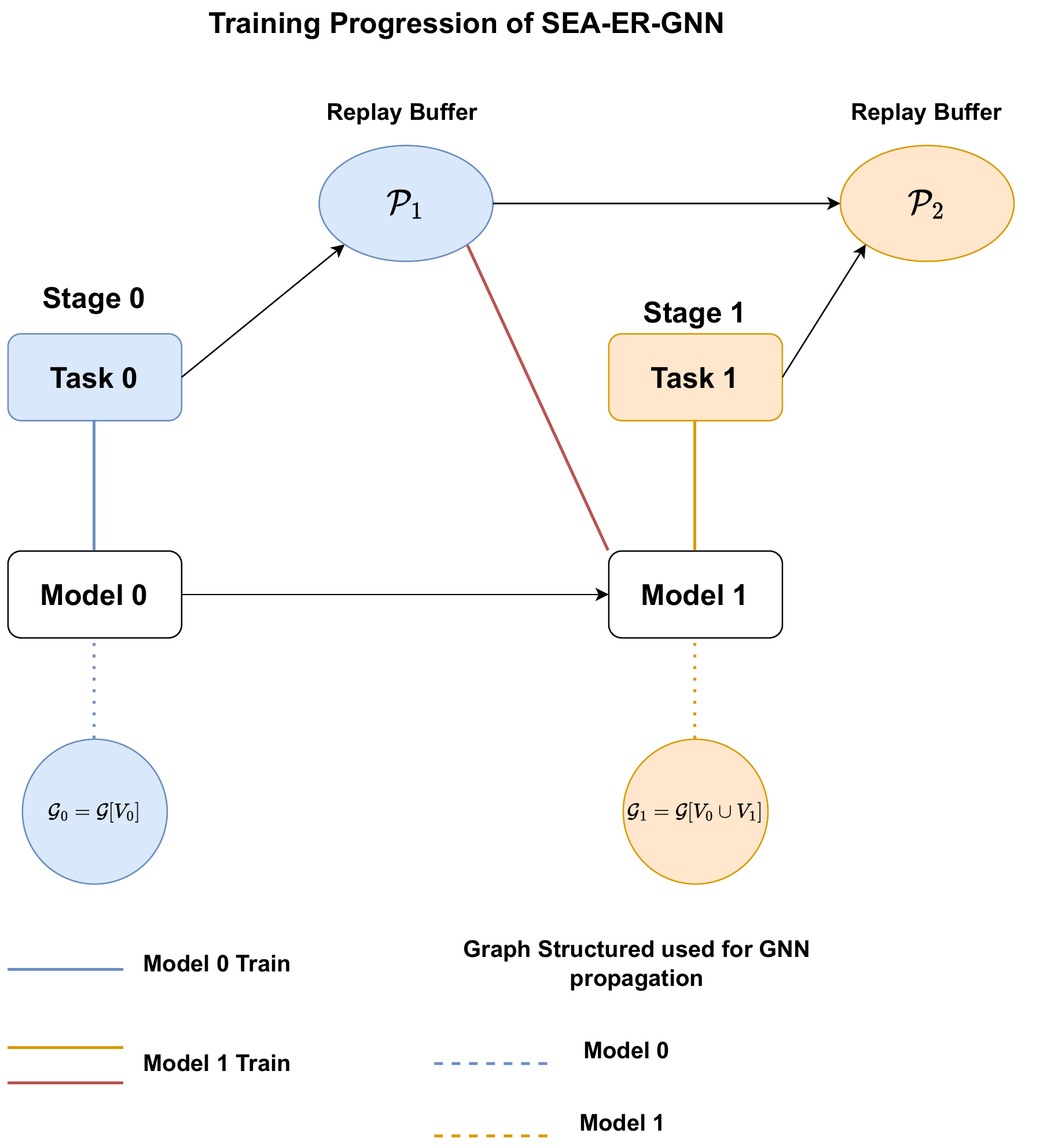}
\caption{Two-stage SEA-ER-GNN Training
}
\label{fig:sail}
\end{figure}

\subsection{Selecting Replay Set}
The replay set selection strategy is given in Eq.~\eqref{eq:experience_buff}. It is obvious that when $D_i = V_i$, the problem can be reduced to the k-center problem~\citep{approx} which is to select a set of k vertices so that the distance from the rest of the vertices to these selected k vertices is minimized. While the k-center problem is NP-hard, there exist efficient greedy approximation algorithms, e.g., by selecting a vertex that is furthest away from the established set~\citep{approx}.  However, such a heuristic might run into the problem of selecting vertices that are at the end of the long path and barely connect to other vertices. To mitigate this, we propose to weight the distance with the degree of the vertices. In addition, like other vertex sampling algorithms, we modify the above selection process by assigning a probability to each vertex based on the greedy criteria and its degree to give the algorithm some room for exploration.

\begin{algorithm}[!h]
  \caption{Replay Set}
  \label{alg:replay_buffer}
\begin{algorithmic}
  \STATE {\bfseries Input:} $D_i \subset V_i$ //labelled set from task i ($\tau_i$)\\
  \REQUIRE $\mathcal{G}_{i-1} = \mathcal{G}[\mathcal{V}_{i-1}]$ //existing graph\\
   
  Sample/select the replay set $P_i$ with the following procedures:
  \begin{enumerate}[leftmargin=.5in]
        \item initialize each vertice $v \in D_i$ with $p_v$ with its degree and sample an initial vertex $u$ based on a probability proportion to the degree (e.g., normalized each the degree of each vertex with the sum of total degree)
        \item assignment $u$ to the replay set $P_i = \{ u \}$
        \item compute the shortest-path distance between each vertex $v \in D_i \setminus P_i$ to $P_i$ with the formula $d(v, P_i) = \min \{d(v,u)| u \in P_i\}$
        \item compute the weight for each vertex $v \in D_i \setminus P_i$ based on product of the distance to $P_i$ and its degree, i.e., $p_v = \text{degree}(v)*d(v, P_i)$
      \item sample $v$ from $D_i$ based on probability proportion to $p_v$, e.g., normalized it with the sum from the remaining vertices of $D_i$
      \item repeat step 2-5 until $|P_i| = k$ 
  \end{enumerate}
\end{algorithmic}
\end{algorithm}

\subsection{Complexity Analysis}
The time complexity of our algorithm is readily apparent from its procedural steps, and it exhibits a computational complexity of O($N$), where $N$ denotes the number of vertices in the graph. Notably, our approach shares a common characteristic with baseline methods, where the problem formulation might inherently be NP-hard to solve. However, both our method and the baselines leverage efficient heuristics to obtain practical solutions. As a result, the time complexity of our proposed methods and the baseline approaches remains similar, with all falling within the order of O(N). This underscores the computational efficiency and feasibility of our methodology in addressing the challenges posed by the problem at hand.

\end{document}